\documentclass{article}

\usepackage{microtype}
\usepackage{graphicx}
\usepackage{subcaption}
\usepackage{booktabs} 
\usepackage{enumitem}
\usepackage{mathtools}
\usepackage{amssymb}
\usepackage{amsthm}
\usepackage[x11names]{xcolor}

\usepackage[most]{tcolorbox}
\tcbuselibrary{skins, theorems}

\newtcolorbox{corollarybox}{
  enhanced,
  colback=blue!5!white,
  colframe=blue!75!black,
  boxrule=1.5pt,
  arc=3pt,                       
  left=8pt, right=8pt, top=8pt, bottom=8pt,
  fonttitle=\bfseries,
  coltitle=white,
  colbacktitle=blue!75!black,    
  title={Corollary [V-only Path Bound]}
}

\newtcolorbox{theorembox}{
  enhanced,
  colback=orange!5!white,
  colframe=orange!75!black,
  boxrule=1.5pt,
  arc=3pt,
  left=8pt, right=8pt, top=8pt, bottom=8pt,
  fonttitle=\bfseries,
  coltitle=white,
  colbacktitle=orange!75!black,
  title={Theoretical Result}
}

\usepackage[preprint]{icml2026}

\theoremstyle{plain}
\newtheorem{theorem}{Theorem}[section]      
\newtheorem{lemma}{Lemma}[section]             
\newtheorem{corollary}{Corollary}[section]     

\theoremstyle{definition}
\newtheorem{definition}{Definition}[section]   

\theoremstyle{remark}
\newtheorem{remark}{Remark}[section]           

\usepackage[textsize=tiny]{todonotes}
\usepackage{hyperref}
\usepackage[capitalize,noabbrev]{cleveref} 

\RequirePackage{silence}
\icmltitlerunning{CRL-VLA: Continual Vision–Language–Action Learning}

\begin{document}

\twocolumn[
  \icmltitle{CRL-VLA: Continual Vision--Language--Action Learning}

  \icmlsetsymbol{equal}{$\dagger$}

  \begin{icmlauthorlist}
    \icmlauthor{Qixin Zeng}{southampton}
    \icmlauthor{Shuo Zhang}{westlake}
    \icmlauthor{Hongyin Zhang}{westlake}
    \icmlauthor{Renjie Wang}{westlake}
    \icmlauthor{Han Zhao}{westlake}
    \icmlauthor{Libang Zhao}{westlake}
    \icmlauthor{Runze Li}{westlake}
    \icmlauthor{Donglin Wang}{westlake,equal}
    \icmlauthor{Chao Huang}{southampton,equal}
  \end{icmlauthorlist}

  \icmlaffiliation{southampton}{University of Southampton, UK}
  \icmlaffiliation{westlake}{Westlake University, China}

  \icmlcorrespondingauthor{Donglin Wang}{wangdonglin@westlake.edu.cn} 
  \icmlcorrespondingauthor{Chao Huang}{Chao.Huang@soton.ac.uk}

  \icmlkeywords{Continual Learning, VLA, Robotics}

  \vskip 0.3in
]

\printAffiliationsAndNotice{\icmlEqualContribution}

\begin{abstract}
Lifelong learning is critical for embodied agents in open-world environments, 
where reinforcement learning fine-tuning has emerged as an important paradigm to enable 
Vision-Language-Action (VLA) models to master dexterous manipulation through environmental interaction. 
Thus, Continual Reinforcement Learning (CRL) is a promising pathway for deploying VLA models in lifelong robotic scenarios, 
yet balancing stability (retaining old skills) and plasticity (learning new ones) remains a formidable challenge for existing methods.
We introduce CRL-VLA, a framework for continual post-training of VLA models with rigorous theoretical bounds.
We derive a unified performance bound linking the stability-plasticity trade-off to 
goal-conditioned advantage magnitude, scaled by policy divergence.
CRL-VLA resolves this dilemma via asymmetric regulation: 
constraining advantage magnitudes on prior tasks while enabling controlled growth on new tasks.
This is realized through a simple but effective dual-critic architecture with novel Goal-Conditioned Value Formulation (GCVF), where a frozen critic anchors semantic consistency and a trainable  estimator drives adaptation. Experiments on the LIBERO benchmark demonstrate that CRL-VLA effectively harmonizes these conflicting objectives, outperforming baselines in both anti-forgetting and forward adaptation.
\end{abstract}

\section{Introduction}
Reinforcement Learning (RL) post-training has become a central paradigm for aligning Vision–Language–Action (VLA) models with complex embodied reasoning and robotic manipulation tasks~\citep{black2024pi06, kim2025fine}. 
By optimizing policies through interaction, RL post-training enables adaptive, goal-directed behaviors beyond what supervised learning can provide. Sustaining such adaptability requires lifelong learning across non-stationary task streams. 
However, in large-scale settings, Continual Reinforcement Learning (CRL) aims to equip agents with the ability to acquire new skills over time without catastrophic forgetting~\citep{sutton2019bitter}. 
This problem fundamentally requires balancing scalability, stability, and plasticity under non-stationary task distributions~\citep{pan2025survey}. 
Despite its importance, CRL for VLA models remains underexplored, particularly in robotic manipulation scenarios involving multi-modal observations and long-horizon dynamics~\citep{Xu2018Reinforced, pan2025survey}.

Most existing CRL methods enforce stability via parameter- or function-space regularization, but they face severe limitations when applied to modern VLA architectures. Experience replay~\citep{Rolnick2018Experience, Abbes2025Revisiting} reuses outdated transitions whose gradients often conflict with those of new tasks, inducing off-policy errors that destabilize continual learning. 
Second-order regularization methods~\citep{kutalev2021stabilizing} are computationally prohibitive due to the scale and dimensionality of multi-modal Transformers~\citep{kim2025fine, Li2024Robust}. 
Recent approaches favor lightweight constraints such as KL regularization~\citep{shenfeld2025rl}, but such transition-level constraints enforce only short-horizon behavioral similarity and fail to preserve long-term task performance~\citep{Korbak2022RL}.
Backbone-freezing strategies~\citep{hancock2025actions} hinge on the fragile assumption that pretrained multi-modal representations remain aligned with language goals; under task shifts, value learning drifts without explicit language conditioning, degrading adaptation and collapsing plasticity~\citep{jiang2025multimodal, guo2025improving}.

In this work, we show that forgetting in continual VLA learning is fundamentally driven by the goal-conditioned advantage magnitude, which directly links policy divergence to performance degradation on prior tasks. This perspective reframes the stability–plasticity dilemma as an asymmetric regulation problem: 
\textit{suppressing advantage magnitudes on previous tasks to ensure stability, while permitting controlled growth on new tasks to enable plasticity.}

Based on this insight, we propose \textbf{CRL-VLA}, a continual learning framework that disentangles stable value semantics from adaptive policy learning via a novel \textbf{Goal-Conditioned Value Formulation (GCVF)}. 
Our method employs a dual-critic architecture in which a frozen critic anchors long-horizon, language-conditioned value semantics, while a trainable critic estimator drives adaptation. Stability is further enforced by combining trajectory-level value consistency with transition-level KL regularization, while Monte Carlo (MC) estimation and standard RL objectives support efficient learning on new tasks without catastrophic forgetting. Our contributions are summarized as follows:
\begin{itemize}
\item We propose \textbf{CRL-VLA}, a principled continual post-training framework for VLA models that identifies the goal-conditioned advantage magnitude as the key quantity governing the stability–plasticity trade-off.

\item CRL-VLA operationalizes this insight via a dual-critic architecture with a novel goal-conditioned value formulation, enabling asymmetric regulation that preserves prior skills while efficiently adapting to new tasks.

\item Extensive experiments show that CRL-VLA consistently outperforms strong baselines on continual VLA benchmarks, achieving superior knowledge transfer and resistance to catastrophic forgetting.
\end{itemize}

\section{Related work}

\subsection{Reinforcement Learning for VLA Models}
VLA models have gained interest as unified robotic manipulation policies, with recent work exploring online RL to address SFT limitations and improve generalization and long-horizon reasoning.
SimpleVLA-RL~\citep{li2025simplevla} builds upon OpenVLA-OFT and GRPO, showing that
online RL substantially improves long-horizon planning under demonstration scarcity. RIPT-VLA~\citep{tan2025interactiveposttrainingvisionlanguageactionmodels} applies the REINFORCE
leave-one-out estimator to QueST and OpenVLA-OFT architectures.
ARFM~\citep{zhang2025balancing} enhances VLA action models via adaptive offline RL, balancing signal and variance to improve generalization and robustness.
RLinf-VLA~\citep{zang2025rlinf} implements a novel hybrid fine-grained pipeline allocation mode scalling up VLA model training with diverse RL algorithms and benchmarks.
Recent work introduces IRL-VLA, a closed-loop RL framework for VLA autonomous driving that achieved award-winning performance in benchmarks~\citep{jiang2025irl} .
In contrast, our work aims to investigate the knowledge transfer of VLA models and their ability to learn new tasks stably and continuously.

\subsection{Continual Reinforcement Learning}
Continual reinforcement learning generally falls into three paradigms. 
Regularization-based methods constrain updates to important parameters using Fisher information~~\citep{DBLP:journals/corr/abs-2105-04093} or path integrals~~\citep{zenke2017continual}, or enforce behavioral consistency via KL divergence~~\citep{shenfeld2025rl}. 
Replay-based strategies~~\citep{Rolnick2018Experience} mitigate forgetting by interleaving past experience, though recent benchmarks~~\citep{wolczyk2021continualworld} highlight their scalability limits in high-dimensional robotics. 
Architecture-based approaches~~\citep{Xu2018Reinforced} prevent interference by expanding network capacity, a strategy often impractical for large-scale pretrained models. 
While limited works explore value preservation via distillation~~\citep{schwarz2018progress,rusu2016policy}, they typically neglect the value drift inherent in non-stationary, language-conditioned tasks. 
Unlike these approaches, our work focuses on the continuous learning process of VLA models for general visual language manipulation tasks.

\subsection{Continual learning in VLA model}
To date, there are no studies on continual learning specifically for VLA models, but similar research methods exist. Recent work~\citep{yadav2025robust} improves policy stability under task variations, providing useful insights for subsequent lifelong learning approaches, although it is evaluated in a non-continual setting. Recently, several works have extended these paradigms to the VLA domain. VLM2VLA~~\citep{hancock2025actions} treats actions as a special language tokens and employs LoRA to align pre-trained VLMs with low-level robot control. Stellar-VLA~~\citep{wu2025continually} introduces a skill-centric knowledge space to evolve task representations continually across diverse robot manipulation scenarios. DMPEL~~\citep{lei2025dynamic} proposes a dynamic mixture of progressive parameter-efficient experts to achieve lifelong learning without high storage overhead. ChatVLA~~\citep{zhou2025chatvla} addresses spurious forgetting by decoupling multi-modal understanding from high-frequency action execution. 
Unlike these work, we aim to investigate a continuous post-training framework for VLA models and explicitly model the stability-plasticity trade-off between new and old tasks.

\section{Preliminaries}
We study continual post-training of VLA policies in a goal-conditioned reinforcement learning setting. Each task is formulated as a goal-conditioned MDP $\mathcal{M}=\langle \mathcal{S},\mathcal{A},\mathcal{G},P,r,\gamma\rangle$, where $s\in\mathcal{S}$ is the state, $g\in\mathcal{G}$ is a language instruction, and $a\in\mathcal{A}$ is a robotic control command. 
For a given policy $\pi$ and goal $g$, we define the state-value function as $V^\pi(s,g) = \mathbb{E}_{\tau\sim\pi,g}[\sum_{t=0}^{\infty}\gamma^t r(s_t,a_t,g) | s_0=s]$, which represents the expected cumulative return starting from state $s$. The action-value function is defined as $Q^\pi(s,a,g) = \mathbb{E}[\sum_{t=0}^{\infty}\gamma^t r(s_t,a_t,g) | s_0=s, a_0=a]$. The advantage function $A^\pi(s,a,g) = Q^\pi(s,a,g) - V^\pi(s,g)$ measures how much better an action is compared to the average action at that state. 

The state occupancy measure $d_g^\pi(s) = (1-\gamma)\mathbb{E}[\sum_{t=0}^{\infty}\gamma^t \mathbf{1}(s_t=s)]$ describes the distribution of states visited by policy $\pi$ when executing goal $g$. Different policies induce different state visitation distributions, which is fundamental to understanding how policy updates affect task performance.
The goal-conditioned policy is $\pi(a\mid s,g)$, with expected return for goal $g$ given by $J_g(\pi)=\mathbb{E}_{\tau\sim\pi,g}\Big[\sum_{t=0}^{\infty}\gamma^t r(s_t,a_t,g)\Big]$.
Performance on old tasks under distribution $p_{\mathrm{old}}(g)$ is $J_{\mathrm{old}}(\pi)=\mathbb{E}_{g\sim p_{\mathrm{old}}}[J_g(\pi)]$. New policy's performance $\pi'$ on new tasks under distribution $p_{\mathrm{new}}(g)$ is $J_{\mathrm{new}}(\pi')=\mathbb{E}_{g\sim p_{\mathrm{new}}}[J_g(\pi')]$. 

A pretrained VLA policy $\pi_{\theta_0}$ is adapted sequentially to a task stream $\mathcal{T}=\{\mathcal{T}_1,\ldots,\mathcal{T}_K\}$. At stage $k$, the agent interacts only with $\mathcal{T}_k$ and updates to $\pi^\theta_k$ via on-policy RL, without access to interaction data, rewards, or gradients from previous tasks $\{\mathcal{T}_i\}_{i<k}$. For a trajectory $\tau = (s_0, a_0, \dots, s_T)$ with goal $g$, the MC return at time $t$ is $G_t= \sum_{k=t}^{T} \gamma^{k-t} r(s_k, a_k, g),$
which provides an unbiased estimate of $V^\pi(s_t, g)$ for critic training.
We evaluate using a transfer matrix $\mathbf{R}\in\mathbb{R}^{K\times K}$, where $R_{k,i}$ is the success rate on task $\mathcal{T}_i$ after training through stage $k$. The objective is to learn a sequence of policies that jointly satisfy three criteria: Plasticity through high performance on the current task, Stability through preserved performance on prior tasks, and Scalability through efficient learning without large buffers or capacity expansion.

\section{Methodology}
In this section, we first characterize the stability--plasticity dilemma in continual VLA through the advantage magnitude and policy divergence, and then demonstrate that $M_{\mathrm{old}}$ and $M_{\mathrm{new}}$ can be controlled in a decoupled manner. 
Furthermore, we introduce a dual goal-conditioned critic for continual VLA and then present the corresponding objectives, regularization, and training recipe.
\subsection{The Stability--Plasticity Dilemma in Continual VLA}
At stage $k$ of continual VLA post-training, we update the policy $\pi^k$ with two coupled criteria. \textbf{Plasticity} means the maximization of new-goal return $J_{g_k}(\pi^k; \mathcal{T}_k)$, while \textbf{stability} is the bounded degradation of old-goal return under the same update. Therefore, we solve
\begin{gather}
\label{objective_1}
\max_{\pi^k}\; J_{g_k}(\pi^k; \mathcal{T}_k) \\
\text{s.t.}\quad J_{g_{k-1}}(\pi^k; \mathcal{T}_{k-1}) \ge J_{g_{k-1}}(\pi^{k-1}; \mathcal{T}_{k-1}) - \delta .
\end{gather}
Here, each task $T_k$ in the sequence $T = \{T_1, \dots, T_K\}$ specifies a goal $g_k$, and the agent iteratively updates its policy $\pi^\theta_k$ via reinforcement learning under a bound $\delta$ on degradation over prior tasks. 
However, directly measuring the return impact of policy divergence during task transition is intractable. We therefore seek a bridge that links policy divergence to return change. 
In continual reinforcement learning for VLA model, we introduce the advantage magnitude $M_g$.
\begin{definition}
\label{def:advantage_magnitude_main}
For any policy $\pi^\kappa$ and goal $g$, we define the \textit{Advantage Magnitude}, denoted as $M_g(\pi^\kappa)$, as the maximum absolute advantage of the anchored policy $\pi$ (typically set as the previous policy $\pi^\mathrm{old}$) evaluated over the state-action pairs visited by $\pi^\kappa$:
\small{
\begin{equation}
    M_g(\pi^\kappa) \triangleq \sup_{(s,a) \in \mathrm{supp}(d_g^{\pi^\kappa})} \big| A^{\pi}(s, a) \big|.
\end{equation}
}
\end{definition}

Here, $M_g(\pi^\kappa)$ close to 0 indicates that the policy aligns closely with the anchored policy on the evaluated distribution.
Based on Definition~\ref{def:advantage_magnitude_main}, we evaluate a post-update policy $\pi^{\mathrm{new}}$ under the old-goal distribution $p_{\mathrm{old}}$ of old task and the new-goal distribution $p_{\mathrm{new}}$ of new task. We then define two metrics to quantify the stability-plasticity trade-off. The \textbf{Stability Metric} $M_{\mathrm{old}}$ measures the maximum advantage of the old policy over its own state distribution, capturing performance retention on old tasks. The \textbf{Plasticity Metric} $M_{\mathrm{new}}$ measures the maximum advantage achievable on new tasks under the new policy:
\begin{equation}
\begin{aligned}
M_{{\mathrm{old}}} &\triangleq \mathbb{E}_{g_{\mathrm{old}} \sim p_{\mathrm{old}}} \left[ M_{g_{\mathrm{old}}}(\pi^{\mathrm{new}}) \right], \\
M_{{\mathrm{new}}} &\triangleq \mathbb{E}_{g_{\mathrm{new}} \sim p_{\mathrm{new}}} \left[ M_{g_{\mathrm{new}}}(\pi^{\mathrm{new}}) \right].
\end{aligned}
\end{equation}
\paragraph{A Unified Perspective on Stability and Plasticity.} 
Theorem~\ref{thm:unified_kl} establishes a unified perspective on the stability-plasticity trade-off in goal-conditioned continual learning. 
Both performance degradation on old tasks and performance improvement on new tasks are governed by the coupling 
mechanism of advantage magnitude and policy divergence measured by KL divergence.

This unified bound suggests that controlling continual learning reduces to balancing these two coupled quantities on old and new tasks. We next discuss how $M_{\mathrm{old}}$ and $D_{\mathrm{old}}$ govern stability, and how $M_{\mathrm{new}}$ and $D_{\mathrm{new}}$ govern plasticity.
\textbf{1)} Old-task stability.$M_{\mathrm{old}}$ characterizes the sensitivity of old-task returns to policy changes, while $D_{\mathrm{old}}$ quantifies the policy change magnitude under the old-task state distribution; the $(1-\gamma)^{-2}$ scaling factor amplifies this coupling in long-horizon settings. 
\textbf{2)} New-task plasticity. $M_{\mathrm{new}}$ reflects the policy improvement potential, while $D_{\mathrm{new}}$ measures the policy modification degree under the new-task state distribution. 
Plasticity requires large policy modifications ($D_{\mathrm{new}}$) combined with learning potential ($M_{\mathrm{new}}$); conversely, constraining $D_{\mathrm{new}}$ to preserve old-task performance limits achievable improvement on new tasks. 

However, the coupling challenge remains stability and plasticity are fundamentally coupled through policy divergence under different state distributions. Common continual learning methods impose global constraints on policy updates via trust-region or KL penalties, restricting policy changes across all state space regions~~\citep{kessler2022same}. 
Reducing $D_{\mathrm{old}}$ to tighten the old-task stability bound also reduces $D_{\mathrm{new}}$, thereby limiting plasticity on new tasks. Therefore, an effective trade-off requires minimizing $M_{\mathrm{old}}$ while constraining policy divergence under both old-task and new-task distributions with controllable $M_{\mathrm{new}}$.

\begin{theorembox}
\begin{theorem}[Unified Stability-Plasticity Bounds]
\label{thm:unified_kl}
Let $\pi^\mathrm{new}$ and $\pi^\mathrm{old}$ be the new and old policies. $J_{\mathrm{old}}(\pi)$ and $J_{\mathrm{new}}(\pi)$ denote the expected returns of policy $\pi$ on the old and new tasks, respectively.  The policy divergence parameters $D_{\mathrm{old}}$ and $D_{\mathrm{new}}$ directly use the expected KL divergence:
\[
D_{\mathrm{old}} \triangleq \sqrt{2\,\mathbb{E}_{s\sim d_{\mathrm{old}}^{\pi^{\mathrm{old}}}}\left[D_{\mathrm{KL}}\!\big(\pi^{\mathrm{new}}(\cdot|s)\,\|\,\pi^{\mathrm{old}}(\cdot|s)\big)\right]},
\]
\[
D_{\mathrm{new}} \triangleq \sqrt{2\,\mathbb{E}_{s\sim d_{\mathrm{new}}^{\pi^{\mathrm{new}}}}\left[D_{\mathrm{KL}}\!\big(\pi^{\mathrm{new}}(\cdot|s)\,\|\,\pi^{\mathrm{old}}(\cdot|s)\big)\right]}.
\]
The performance variations are bounded by:
\[
\big|J_{\mathrm{old}}(\pi^{\mathrm{new}})-J_{\mathrm{old}}(\pi^{\mathrm{old}})\big| \le \frac{2\gamma}{(1-\gamma)^2} \cdot M_{{\mathrm{old}}} \cdot D_{\mathrm{old}},
\]
\[
J_{\mathrm{new}}(\pi^{\mathrm{new}})-J_{\mathrm{new}}(\pi^{\mathrm{old}})\le\frac{1}{1-\gamma} \cdot M_{{\mathrm{new}}} \cdot D_{\mathrm{new}}.
\]
\end{theorem}
\end{theorembox}

\begin{proof}
Using Performance Difference Lemma \citep{kakade2002approximately} 
and discounted occupancy bounds \citep{achiam2017constrained}, 
we derive performance difference bounds between old and new policies 
under goal-conditioned task transitions. 
Full derivations are in Appendix~\ref{app:thereom_deri}.
\end{proof}

\begin{remark}
An effective stability-plasticity trade-off requires minimizing $M_{\mathrm{old}}$ while constraining $D_{\mathrm{old}}$ and $D_{\mathrm{new}}$ with controllable $M_{\mathrm{new}}$.
\end{remark}
\begin{figure*}[t]
    \centering
    \includegraphics[width=0.9\textwidth]{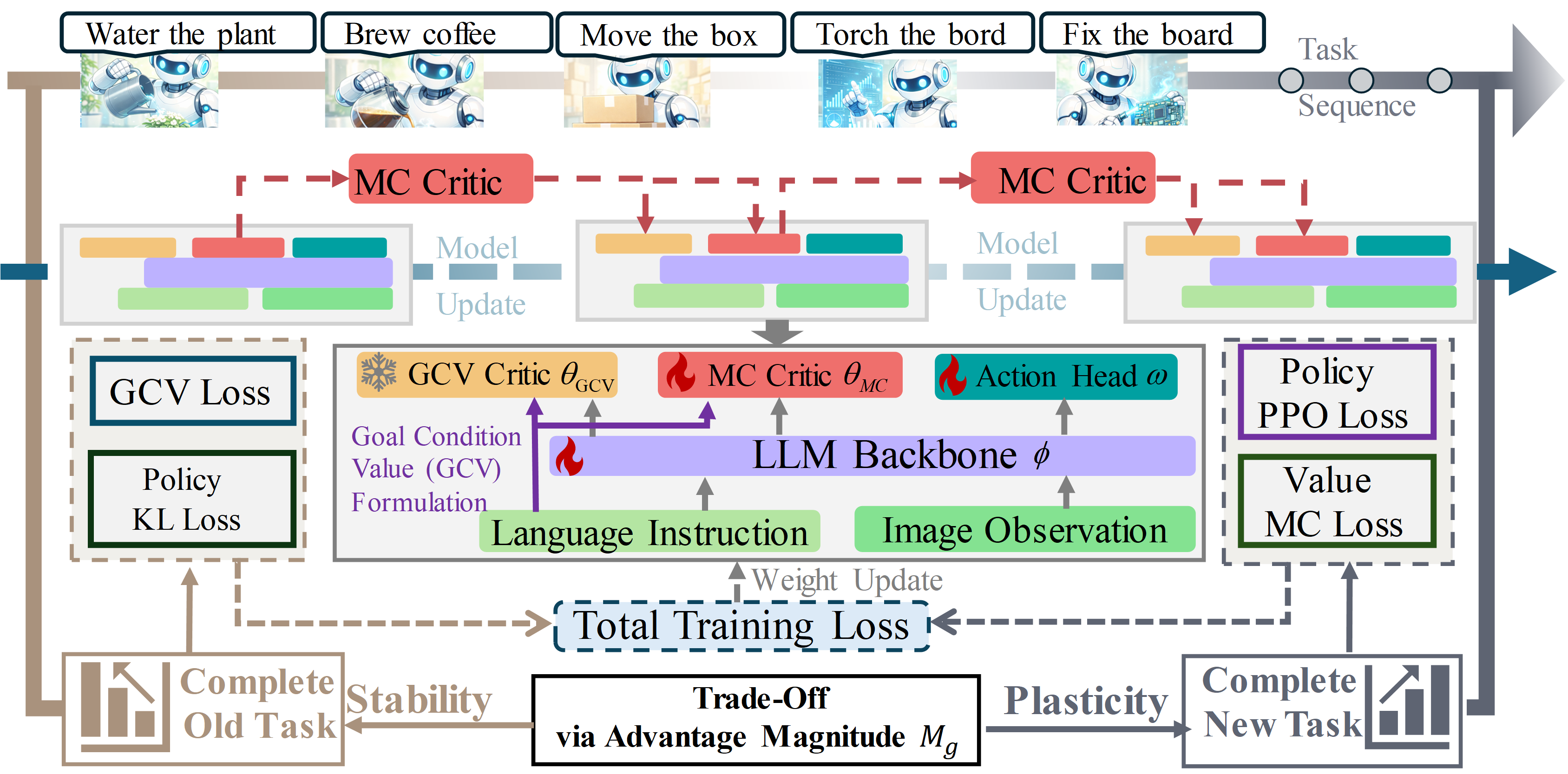}
    \caption{The proposed CRL-VLA method.
CRL-VLA is a continuous learning framework for VLA models that treats the goal-conditioned advantage magnitude as a key factor determining the trade-off between stability and plasticity.
CRL-VLA achieves this insight through a dual critic architecture and a PPO loss with several regularization terms, enabling asymmetric tuning that efficiently adapts to new tasks while preserving prior knowledge.
}
\label{fig:CRL_VLA}
\end{figure*}
\subsection{Decouple control for Stability and Plasticity Metrics}


The central question becomes: can we control stability metrics $M_{\mathrm{old}}$ and plasticity metrics $M_{\mathrm{new}}$ separately? The answer depends on understanding the fundamental differences between these two quantities.
Corollaries~\ref{cor:v_only_bound_main}-\ref{cor:mc_advantage_bound_main} reveal theses two metrics are bounded by independent factors with concrete derivations in Appendix~\ref{app:corollaries}. $M_{\mathrm{old}}$ is constrained by critic approximation error
$\varepsilon_V$ on $\mathcal{D}_{\mathrm{old}}$, while $M_{\mathrm{new}}$ is limited by environment return range $[G_{\min}, G_{\max}]$. Since $\varepsilon_V$ operates on replay data and environment return range is task-intrinsic, a single constraint cannot regulate both. This necessitates dual mechanisms: minimizing $\varepsilon_V$ via constrained updates on $\mathcal{D}_{\mathrm{old}}$ controls $M_{\mathrm{old}}$, while MC-based estimation on new tasks exploits natural boundedness to control $M_{\mathrm{new}}$. The orthogonality of these mechanisms enables simultaneous stability and plasticity without trade-offs: we control $M_{\mathrm{old}}$ by minimizing $\varepsilon_V$ on old task data, while $M_{\mathrm{new}}$ remains naturally bounded by the environment's return range during new task learning.



\begin{corollarybox}
\begin{corollary}[Controllability of Stability]
\label{cor:v_only_bound_main}
Let $V_\theta$ denote the parameterized approximation and $V_{\mathrm{old}}$ the true anchor value. Under bounded rewards $|r(s,a,g_{\mathrm{old}})| \le R_{\max}$ and error $\varepsilon_V \triangleq \sup_{s,g_{\mathrm{old}}} |V_\theta - V_{\mathrm{old}}|$, the advantage magnitude on old tasks satisfies:
$
M_{\mathrm{old}} \le R_{\max}+(1+\gamma)\varepsilon_V +(1+\gamma)\|V_{\mathrm{old}}\|_{\infty}.
$
\end{corollary}
\end{corollarybox}
\begin{remark}
The bound reveals that $M_{\mathrm{old}}$ can be directly controlled by constraining the value approximation error $\varepsilon_V$ on old tasks.
\end{remark}

\begin{corollarybox}
\begin{corollary}[Natural Boundedness of Plasticity]
\label{cor:mc_advantage_bound_main}
Let $G_t^g$ be the ground-truth MC return for goal $g$, with $G_t^g \in [G_{\min}, G_{\max}]$, at any time $t$.
Then, for the V-only advantage estimator $\hat A_{g_{\mathrm{new}}}(s,a)$, the advantage magnitude $M_{\mathrm{new}}$ satisfies:
$
M_{\mathrm{new}} \le 2(1 + \gamma) \max\{|G_{\min}|, |G_{\max}|\}.
$
\end{corollary}
\end{corollarybox}

\begin{remark}
The bound reveals that $M_{\mathrm{new}}$ can be directly controlled by bounding by reward of MC estimator.
\end{remark}

\subsection{Dual Goal-Conditioned Critic for continual VLA}
Given that $D_{\mathrm{old}}$ and $D_{\mathrm{new}}$ in Theorem~\ref{thm:unified_kl} are realized via KL divergence, we develop value functions that make advantage magnitudes $M_g$ controllable, serving as a surrogate for direct advantage-based optimization. Furthermore, we enhance language-following capacity by conditioning the value network directly on language embeddings. Inspired by Corollaries~\ref{cor:v_only_bound_main} and~\ref{cor:mc_advantage_bound_main}, we propose a simple yet effective dual-critic framework.

\paragraph{Infeasibility of Direct Advantage Optimization in VLA}
Obtaining accurate advantage estimates in VLA settings requires expensive multi‑step rollouts and bootstrapping, which is hard to scale~~\citep{Chen2023Adaptive}. Concretely, for large VLA policies deployed on real robots or high‑fidelity simulators, each additional rollout entails a full forward pass of a vision–language backbone plus action decoding, making per‑update advantage estimation extremely costly~~\citep{Wen2024TinyVLA:}. 
Therefore, in continual VLA settings, MC estimation is preferred the TD bootstrapping due to prohibitive computational costs in scaling models and the bias inherent in value estimation.
\paragraph{Indirect Control via Value Functions}
Corollaries~\ref{cor:v_only_bound_main} and~\ref{cor:mc_advantage_bound_main} demonstrate that the theoretical upper bounds on the performance gap can be expressed through value functions. 
Guided by this insight, we control $M_g$ indirectly via value estimation rather than imposing direct constraints on the advantage network. 
While value estimation can be realized through either state-value functions or action-value functions, we focus on the state-value function $V(s,g)$ in the main text for conciseness; the extension to action-value functions is detailed in Appendix~\ref{Q_based Implementation}. 
Specifically, for a given VLA policy $\pi$ and language goal $g$, we approximate the advantage function $\hat{A}_g(s,a)$ as:
$\hat{A}_g(s,a) = r(s,a,g) + \gamma V_\pi(s',g) - V_\pi(s,g)$,
where $s \sim d_g^\pi$ is a state sampled from the goal-conditioned state distribution, $s' \sim \mathcal{P}(\cdot|s,a)$ is the successor state, $V_\pi(\cdot,g)$ is the state-value function for goal $g$ under policy $\pi$, and $\gamma \in (0,1)$ is the discount factor.

\paragraph{Goal-Conditioned Value Formulation}
To condition the value function on language goal $g$, we leverage the shared Vision-Language Model's embedding as the global state representation. Specifically, we concatenate language embedding with global state representations before the MLP (Figure~\ref{fig:CRL_VLA}), which outputs the scalar value $V(s,g)$. This joint encoding addresses the poor language-goal following capacity in VLA value heads~~\citep{intelligence2025pi05visionlanguageactionmodelopenworld}.

\paragraph{Dual-Critic Architecture for Independent Dimension Management}
To operationalize indirect control of $M_g$, we design a dual-critic architecture as follows.
\begin{enumerate}
\item Frozen Goal-Conditioned Value (GCV) Critic $\theta_{\text{GCV}}$.
Initialized from the old task's value network and frozen during new task training, it provides the reference value $V_{\text{old}}(s,g)$ for Corollary~\ref{cor:v_only_bound_main}, enabling regularization of value accuracy differences on old data while maintaining predictable bounds via value drift prevention.
\item Trainable MC critic $\theta_{\text{MC}}$. 
Updated on new task MC returns to implement Corollary~\ref{cor:mc_advantage_bound_main}, allowing natural emergence of $M_{\mathrm{new}}$ bounds. 
\end{enumerate}
Overall, the shared backbone $\phi$ serves as a unified feature extractor, while head decoupling enables targeted optimization. Specifically, the frozen 
GCV critic $\theta_{\text{GCV}}$. preserves value consistency and constrains value approximation error, 
guiding the action head $\omega$ to favor high-return actions in old task states. While the trainable MC critic $\theta_{\text{MC}}$ bounds growth while maintaining plasticity for continual learning.

\subsection{Regularization and Training Recipe}
We instantiate Theorem~\ref{thm:unified_kl} with Corollaries~\ref{cor:v_only_bound_main} and~\ref{cor:mc_advantage_bound_main} through three mechanisms that balance stability and plasticity by controlling $M_{\mathrm{old}}$, $M_{\mathrm{new}}$, $D_{\mathrm{new}}$, and $D_{\mathrm{old}}$. This unified algorithm applies to both state and action value functions, with the latter detailed in Appendix~\ref{Q_based Implementation}.

First, we control $M_{\mathrm{old}}$ through Goal-Conditioned Value (GCV) consistency.
At task transition, the converged MC critic is copied to $\theta_{\text{GCV}}$ and frozen. During new task training, backbone updates $\phi$ are penalized when $V_{\phi, \theta_{\text{GCV}}}$ deviates from $V_{\text{old}}$. 
At task transition, we freeze the converged MC  critic as $\theta_{\text{GCV}}$ anchored to reference values $V_{\text{old}}(s,g)$ from $\mathcal{B}_{\text{old}}$ to minimize $\varepsilon_V$ on the old task distribution, as required by Corollary~\ref{cor:v_only_bound_main}. 
We constrain $M_{\mathrm{old}}$ by ensuring the new policy generates trajectory-level action distributions consistent with the old policy when evaluated on sampled states $(s,g) \in \mathcal{B}_{\text{old}}$ from previous tasks. This is achieved through the following loss:
\begin{equation}
\label{GCV-regu}
\mathcal{L}^{V}_{\mathrm{GCV}}(\phi) = \beta_V \mathbb{E}_{(s,g)\sim\mathcal{B}_{\text{old}}} \left[ \left\| V_{\phi, \theta_{\text{GCV}}}(s,g) - V_{\text{old}}(s,g) \right\|^2 \right].
\end{equation}
Second, we realize bounded $M_{\mathrm{new}}$ via MC Critic learning.
Since the  Corollary~\ref{cor:mc_advantage_bound_main} directly shows controllable $M_{\mathrm{new}}$ is bounded by MC return $G_t^{g_{\text{new}}}$ during roll-out. Combined with dense reward shaping~~\citep{zhang2025reinbot}, this achieves accurate advantage estimation without constraining backbone plasticity. Thus, we train a learnable MC Critic critic $\theta_{\text{MC}}$ on new task trajectories:
\begin{equation}
\mathcal{L}^{V}_{\mathrm{MC}}({\theta_{\text{MC}}},\phi) =  \mathbb{E}_{(s,g)\sim\mathcal{B}_{\mathrm{new}}} \left[ \left\| V_{\phi,\theta_{\text{MC}}}(s,g) - G_t^{g_{\text{new}}} \right\|^2 \right].
\end{equation}

Our third mechanism constrains policy divergence $D_{\mathrm{new}}$ and $D_{\mathrm{old}}$ to control distribution shift. 
PPO's trust region constraint (Eq.~\ref{PPO_KL_constraint}, 
Appendix~\ref{PPO appendix}) limits the $\pi$ updates in the distribution of new task, thereby bounding 
$D_{\mathrm{new}}$ drift. To maintain $D_{\mathrm{old}}$ 
(Theorem~\ref{thm:unified_kl}), we propose
\begin{equation}
\label{KL_constrain}
\mathcal{L}_{\mathrm{KL}}(\phi,\omega) = \mathbb{E}_{(s,g)\sim\mathcal{B}_{\text{old}}} 
\left[ D_{\mathrm{KL}}(\pi_{\text{old}}(\cdot\mid s,g) \| \pi(\cdot\mid s,g)) \right],
\end{equation}
where $\pi_{\text{old}}$ is the old policy that generated $\mathcal{B}_{\text{old}}$.
%
The total training loss integrates these constraints:
\begin{equation}
\mathcal{L}_{\mathrm{total}} = \mathcal{L}_{\mathrm{PPO}} + \alpha\,\mathcal{L}_{\mathrm{KL}} + \beta_V\,\mathcal{L}^{V}_{\mathrm{GCV}} + \eta\,\mathcal{L}^{V}_{\mathrm{MC}}.
\label{eq:loss_total}
\end{equation}
This Lagrangian relaxation of Eq.~\eqref{objective_1} instantiates Theorem~\ref{thm:unified_kl}'s dual constraints, with concrete derivations in appendix \ref{Objectives and Total Training Loss}. $\mathcal{L}_{\mathrm{PPO}}$ and $\mathcal{L}_{\mathrm{KL}}$ bound $D_{\mathrm{new}}$ and $D_{\mathrm{old}}$ (policy divergence), while $\mathcal{L}^{V}_{\mathrm{GCV}}$ and $\mathcal{L}^{V}_{\mathrm{MC}}$ bound $M_{\mathrm{old}}$ and enable $M_{\mathrm{new}}$ exploitation (value accuracy and plasticity). Hyperparameters $\alpha, \beta_V, \eta$ act as Lagrange multipliers controlling the stability-plasticity tradeoff.
Overall, the CRL-VLA pipeline has been summarized in Alg.~\ref{algo:crl-vla}.
\begin{algorithm}[htbp]
\caption{CRL-VLA: VLA Policy Continuous Learning}                \label{algo:crl-vla}                  
\begin{algorithmic}[1]
\REQUIRE Pretrained VLA policy $\pi_{\theta_0}$, tasks $\{\mathcal{T}_k\}_{k=1}^K$.
\STATE \textbf{for} $k = 1$ to $K$:
    \STATE \hspace{1em} \textbf{if } $k = 1$:
        \STATE \hspace{2em} Train $\pi_{\theta_1}$ on $\mathcal{T}_1$ using standard RL
        \STATE \hspace{2em} Freeze $\pi_{\theta_1}$ and its value function as $\pi^\mathrm{old}$, $V_{\mathrm{old}}$
        \STATE \hspace{2em} Store replay buffer $\mathcal{B}_{\mathrm{old}}$
    \STATE \hspace{1em} \textbf{else}: 
        \STATE  \hspace{2em} \textbf{for} each training iteration:
            \STATE \hspace{3em} Collect new task trajectories $\tau \sim \pi_\theta$ on $\mathcal{T}_k$
            \STATE \hspace{3em} Sample batch $\mathcal{B}_{\mathrm{new}} \sim \tau$, $\mathcal{B}_{\mathrm{old}} \sim \mathcal{D}_{\mathrm{old}}$
            \STATE \hspace{3em} Compute losses $\mathcal{L}_{\mathrm{total}}$ by Eq.~\ref{eq:loss_total}
            \STATE \hspace{3em} Update $\theta$ via gradient descent on $\mathcal{L}_{\mathrm{total}}$
\STATE Continually adapted VLA policies $\{\pi_{\theta_k}\}_{k=1}^K$
\end{algorithmic}
\end{algorithm}

\section{Experiments}
In this section, we explore how the CRL-VLA effectively achieves a balance between stability and plasticity, thereby enhancing the robot's continuous learning ability for visual language manipulation tasks.
To this end, our experiments aim to investigate the following questions: 
\textbf{1)} Compared to baseline algorithms, does CRL-VLA exhibit better adaptability and resistance to forgetting when performing a single long-horizon task?
\textbf{2)} Compared to baseline algorithms, does CRL-VLA exhibit superior cross-task knowledge sharing and transfer performance when performing multiple long-horizon tasks?
\textbf{3)} How important are the core components of the proposed method to the overall learning performance of the CRL-VLA framework?

\subsection{Setup}
To comprehensively evaluate the performance of the proposed CRL-VLA, we constructed a set of VLA model continuous learning task settings based on the LIBERO~~\citep{liu2023libero} benchmark. 
Specifically, we constructed benchmark subsets by randomly sampling tasks from the LIBERO shared task pool. 
These subsets contain tasks randomly sampled from the original dataset to cover different task scales. 
Each benchmark contains multiple sets of tasks defined by language instructions.
See Appendix~\ref{appendix:Task and Experiment Details} for specific task settings. 
Furthermore, we used the same OpenVLA-oft model and PPO post-training configuration for all algorithms~~\citep{tan2025interactiveposttrainingvisionlanguageactionmodels, kim2025fine}. 
To systematically answer the first and second questions, two types of experimental settings were primarily considered:
\textbf{1)} Single-task learning scenarios: used to evaluate the adaptability of each algorithm to a specific task;
\textbf{2)} Multi-task learning scenarios: used to evaluate the cross-task knowledge sharing and continuous learning performance of each algorithm.

\paragraph{Baselines.}
To thoroughly compare the superiority of the CRL-VLA, we consider several classic CL algorithms.
\textbf{1)} Sequence Learning (SL)~~\citep{liu2023libero}: This is the most direct CL algorithm. 
The VLA model is only fine-tuned for new tasks, without any specific mechanism to prevent forgetting.
\textbf{2)} Learning Without Forgetting (LWF)~~\citep{li2017learning}: This baseline algorithm retrains the VLA model using only new task data while preserving the model's original functionality.
\textbf{3)} Experience Replay (ER)~~\citep{lopez2017gradient}: This baseline algorithm stores sample data from past tasks and converts them into gradient constraints, achieving gradient contextual memory.
\textbf{4)} Multi-Task Learning (MTL)~\citep{liu2023libero}: This baseline algorithm learns from both new and old tasks simultaneously. 
Our method includes two versions: one using V-based goal-conditioned value (\textbf{CRL-VLA (V)}) and the other using Q-based goal-conditioned value (\textbf{CRL-VLA (Q)}).

\paragraph{Evaluation Metrics.}
To systematically evaluate the performance of each algorithm, we adopted three standard metrics widely used in CRL benchmarks~\citep{wolczyk2021continualworld}: 
\textbf{1) Final Performance}; 
\textbf{2) Stability and Forgetting}; 
\textbf{3) Plasticity and Forward Transfer}.
Specifically, the \textbf{Final Performance} metric represents the overall task performance of the VLA model after continuous training, and is measured by the Final Average Return (\textbf{FAR}):
$
\mathrm{FAR} := \frac{1}{T} \sum_{i=1}^{T} R_{T,i}.
$
This metric reflects the average success rate of the VLA model across all tasks in the post-training phase.
\textbf{Stability and Forgetting} metrics are evaluated using Backward Transfer (\textbf{BWT}), which measures how learning a new task affects the performance of previously learned tasks:
$
\mathrm{BWT} := \frac{1}{T-1} \sum_{i=1}^{T-1} \left( R_{T,i} - R_{i,i} \right).
$
A negative BWT indicates a catastrophic forgetting rate in the VLA model.
To explicitly quantify forgetting, we further define the Forgetting (\textbf{F}) metric for the VLA model:
$
F_i := \max_{k \ge i} R_{k,i} - R_{T,i},
\mathrm{F} := \frac{1}{T-1} \sum_{i=1}^{T-1} F_i ,
$
This metric measures the degree of degradation from peak performance to final performance for each task.
\textbf{Plasticity and Forward Transfer} metrics are evaluated using Forward Transfer (\textbf{FT}), which reflects the extent to which prior knowledge facilitates the learning of new tasks.
Let $b_i$ represent the baseline success rate of task $\mathcal{T}_i$ before learning previous tasks.
We define: 
$\mathrm{FT}:=\frac{1}{T-1}\sum_{i=2}^{T}\left(R_{i-1,i}-b_i\right),
$
where a positive value indicates a beneficial forward transition in the VLA model.

\begin{table}[htbp]
\centering
\setlength{\tabcolsep}{3pt}        
\caption{Performance comparison on the  single-task learning scenario.
Specific task settings in Appendix~\ref{para:task1}.
$\uparrow$ denotes higher is better, $\downarrow$ denotes lower is better. Best results are \textbf{bold}, and \underline{second best} are underlined.}
\label{tab:spatial1_results_compact}
\begin{scriptsize} 
\begin{tabular}{lcccc}
\toprule
\textbf{Method} & \textbf{FAR} $(\uparrow)$ & \textbf{BWT} $(\uparrow)$ & \textbf{FT} $(\uparrow)$ & \textbf{F} $(\downarrow)$ \\
\midrule
SL~\citep{liu2023libero}            & 0.00          & -0.62          & -0.50          & 0.62 \\
MTL~\citep{liu2023libero}           & \underline{0.96} & \textbf{0.11}  & \textbf{-0.06} & \underline{0.06} \\
ER~\citep{lopez2017gradient}        & 0.60          & -0.51          & \textbf{-0.06} & 0.53 \\
LWF~\citep{li2017learning}          & 0.67          & -0.50          & \textbf{-0.06} & 0.50 \\
\midrule
CRL-VLA (V)                         & 0.67          & -0.49          & \textbf{-0.06} & 0.50 \\
CRL-VLA (Q)                         & \textbf{0.98} & \underline{-0.02} & \textbf{-0.06} & \textbf{0.03} \\
\bottomrule
\end{tabular}
\end{scriptsize}
\end{table}

\subsection{Main Results}
\paragraph{Single-task learning scenario.}
Tab.~\ref{tab:spatial1_results_compact} compares the continuous learning performance of various algorithms in single-task learning scenarios. 
The results show that, compared to the baselines, the CRL-VLA (Q) has the highest FAR metric value ($0.98$), the second highest BWT metric value ($-0.02$), and the lowest F metric value ($0.03$). 
Except for the SL baseline algorithm ($-0.5$), the other methods exhibit similarly high FT metric values ($-0.06$). 
Therefore, our method has better overall task success rate and resistance to forgetting old tasks, thus enabling more continuous and stable learning of new tasks. 
This also confirms that our method allows the VLA model to better balance task stability and plasticity during continuous learning.

\begin{table}[htbp]
\centering
\caption{
Performance comparison on the multi-task learning scenario. 
Specific task settings in Appendix~\ref{para:task2}. 
}
\label{tab:task3_results}
\setlength{\tabcolsep}{6pt} 
\begin{scriptsize} 
\begin{tabular}{lccc}
\toprule
\textbf{Method} & \textbf{FAR} $(\uparrow)$ & \textbf{BWT} $(\uparrow)$ & \textbf{FT} $(\uparrow)$ \\
\midrule
SL~\citep{liu2023libero}     & 0.62          & 0.07            & \textbf{0.25} \\
MTL~\citep{liu2023libero}    & 0.49          & 0.00            & \textbf{0.25} \\
ER~\citep{lopez2017gradient} & 0.62          & 0.05            & \underline{0.00} \\
LWF~\citep{li2017learning}   & 0.63          & \underline{0.10} & \underline{0.00} \\
\midrule
CRL-VLA (V)                  & \textbf{0.74} & \textbf{0.17}   & \underline{0.00} \\
CRL-VLA (Q)                  & \underline{0.66} & -0.03           & \underline{0.00} \\
\bottomrule
\end{tabular}
\end{scriptsize}
\end{table}
\begin{figure}[htbp]
    \centering
    \includegraphics[width=0.9\columnwidth]{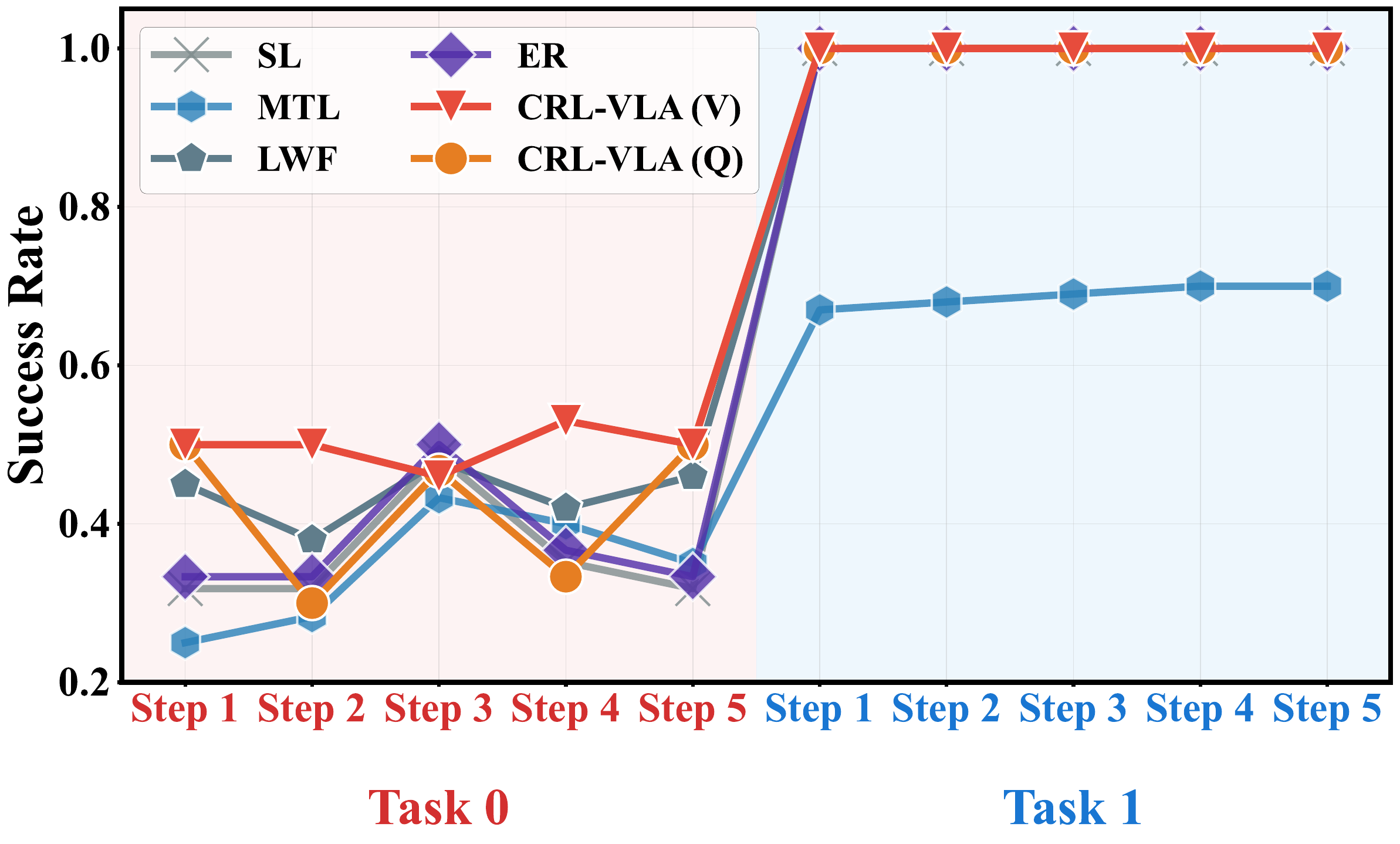}
    \caption{
    Performance comparison of various algorithms during continual learning.
    Specific task settings in Appendix~\ref{para:task2}.
    }
    \label{fig:crl_learning_process_comp}
\end{figure}

\begin{figure*}[htbp]
\centering
    \includegraphics[width=0.9\linewidth]{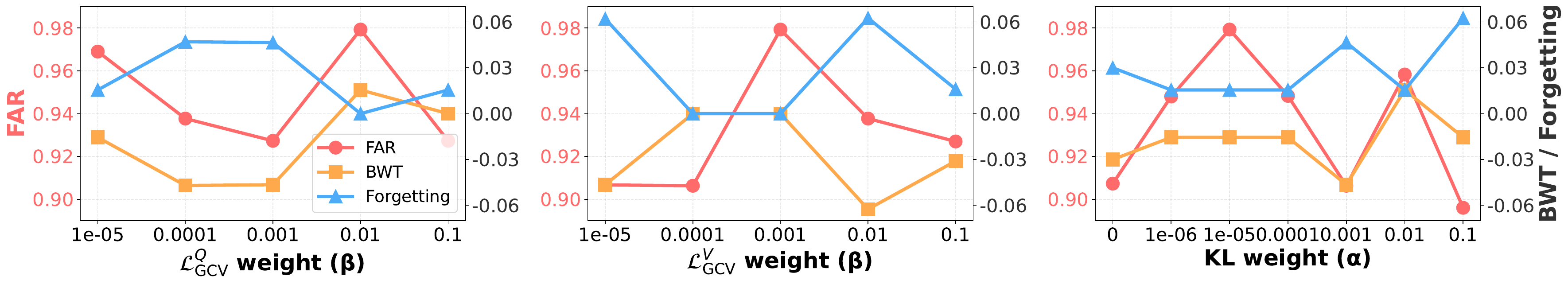}
    \caption{The impact of different values of the loss weights of $\mathcal{L}^{Q}_{\mathrm{GCV}}$ (left), $\mathcal{L}^{V}_{\mathrm{GCV}}$ (middle), and policy KL (right) on the continuous learning performance of the VLA model. Specific task settings in Appendix~\ref{para:task4}.
    }
    \label{fig:main_ablation}
\end{figure*}
\begin{figure*}[htbp]
\centering
    \includegraphics[width=0.9\linewidth]{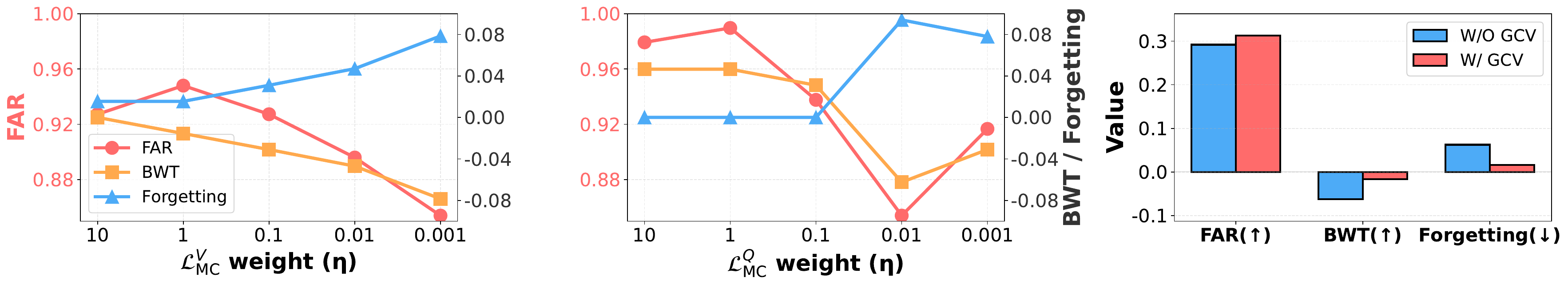}
    \caption{The impact of different values of the loss weights of $\mathcal{L}^{V}_{\mathrm{mc}}$ (left) and $\mathcal{L}^{Q}_{\mathrm{mc}}$ (middle) on the continuous learning performance of the VLA model, and an ablation comparison with and without GCV (right). Specific task settings in Appendix~\ref{para:task4}.
    }
    \label{fig:value_loss_mc}
\end{figure*}
\paragraph{Multi-task learning scenario.}
To further examine whether the proposed method CRL-VLA possesses better cross-task knowledge sharing and continuous learning capabilities in multi-task learning scenarios, we conducted a performance comparison experiment (Tab.~\ref{tab:task3_results}). 
The experimental results show that, compared to the baseline algorithms, our method CRL-VLA (V) has the highest FAR metric value ($0.74$) and the highest BWT metric value ($0.17$). 
Furthermore, regarding the FT metric values, SL and MTL have the highest values ($0.25$), while other methods have the same value ($0.$). 
Therefore, in multi-task learning scenarios, compared to the baseline, our method possesses better cross-task knowledge sharing and anti-forgetting capabilities, thus exhibiting better continuous learning capabilities. 
Moreover, the experiment result also demonstrates that our method can significantly improve the overall task success rate of the VLA model, and the performance comparison of each algorithm during the continuous learning process in Fig.~\ref{fig:crl_learning_process_comp} further confirms this. 
For more performance comparisons in multi-task learning scenarios, see Appendix Tab.~\ref{tab:object2_results}.

Furthermore, experimental results (Tab.~\ref{tab:spatial1_results_compact} and Tab.~\ref{tab:task3_results}) also show that $\mathcal{L}^{Q}_{\mathrm{GCV}}$ is more effective in single-task learning, while $\mathcal{L}^{V}_{\mathrm{GCV}}$ performs better in multi-task continuous learning environments. 
We believe the fundamental reason for this is that in multi-task continuous learning, the state-goal distribution becomes more diverse, and similar states may correspond to multiple effective action implementations in different tasks. 
$\mathcal{L}^{Q}_{\mathrm{GCV}}$ constrains value consistency at the action-condition level, implicitly enforcing outdated action semantics sampled from previous tasks, which may conflict with the current learning objective. 
In contrast, $\mathcal{L}^{V}_{\mathrm{GCV}}$ operates at the state-goal level and marginalizes actions, thus maintaining robustness to increased action diversity while preserving long-horizon value semantics.

\subsection{Ablation Study}
We first conduct ablation experiments on the GCV loss weights and policy KL loss weights in the CRL-VLA to evaluate the contribution of the GCV and policy KL components to the continuous learning performance of VLA. 
The experimental results are shown in Fig.~\ref{fig:main_ablation}.
For $\mathcal{L}^{Q}_{\mathrm{GCV}}$ and $\mathcal{L}^{V}_{\mathrm{GCV}}$ losses, moderate loss weights (approximately 0.01 and 0.001 respectively) yield better FAR and BWT, as well as lower forgetting, achieving a better post-training stability-plasticity balance for the VLA model. 
Furthermore, excessively small weights fail to suppress $M_{{\mathrm{old}}}$, while excessively large weights over-constrain updates and reduce forward adaptation.
On the other hand, moderate policy KL term weights $\alpha$ perform best (approximately 0.01). An excessively strong KL term narrows the feasible update region and indirectly inhibits the effective use of $M_{{\mathrm{new}}}$, while an excessively weak KL term leads to policy drift, potentially increasing the actual $M_{{\mathrm{old}}}$ even with critic anchoring.

We further conduct ablation experiments on the value loss MC weights (Fig.~\ref{fig:value_loss_mc}, left and middle).
The results show that the VLA model exhibits the best sustained learning performance when the weights of $\mathcal{L}^{V}_{\mathrm{MC}}$ and $\mathcal{L}^{Q}_{\mathrm{MC}}$ are approximately $1$.
When the value loss weight is too small, the judge error increases, leading to a larger actual advantage and consequently worsening the BWT and forgetting effects; conversely, the optimization process tends to be conservative, potentially reducing the final task performance. 
This trend confirms that the quality of the critic estimation is a direct means of adjusting $M_g$, which directly affects the sustained learning performance of the VLA model. 
Furthermore, experiments comparing the effects of GCV ablation (Fig.~\ref{fig:value_loss_mc}, right) demonstrate that directly allowing the critic to perceive the language goal brings better success rates, resistance to forgetting, and plasticity to the VLA model.

\section{Conclusion and Future Work}
In this work, we study continual learning for VLA models and identify the goal-conditioned advantage magnitude as the key factor governing the stability–plasticity trade-off. 
Building on this theoretical insight, we propose CRL-VLA, an asymmetric regularization framework realized through a dual-critic architecture with a novel goal-conditioned value formulation.
Empirical results demonstrate that CRL-VLA effectively mitigates catastrophic forgetting while maintaining strong forward transfer, outperforming existing baselines. 
An important direction is extending CRL-VLA to more diverse and unstructured task streams, including partial observability and non-stationary language goals. 

\bibliographystyle{icml2026}
\bibliography{example_paper}

\onecolumn
\appendix
\section{Theoretical Analysis and Proofs}
\label{sec:appendix}

\subsection{Preliminaries and Definitions}

\begin{definition}[Mixture Occupancy Measures]
\label{def:mixture_occupancy}
Given a goal distribution $p(g)$, we define the mixture occupancy measure as:
\[
d_p^{\pi}(s) \triangleq \mathbb{E}_{g\sim p}[d_g^{\pi}(s)],
\]
where $d_g^{\pi}(s) = (1-\gamma)\sum_{t=0}^{\infty} \gamma^t \Pr(s_t = s \mid g, \pi) $ is the discounted state visitation distribution.
\end{definition}

\begin{definition}[Advantage Magnitude]
\label{def:advantage_magnitude}
For any policy $\pi^\kappa$ and goal $g$, we define the \textit{Advantage Magnitude}, denoted as $M_g(\pi^\kappa)$, as the maximum absolute advantage of the anchored policy $\pi$ (typically set as the previous policy $\pi^\mathrm{old}$) evaluated over the state-action pairs visited by $\pi^\kappa$:
\begin{equation}
    M_g(\pi^\kappa) \triangleq \sup_{(s,a) \in \mathrm{supp}(d_g^{\pi^\kappa})} \big| A^{\pi}(s, a) \big|.
\end{equation}
Based on Definition~\ref{def:advantage_magnitude}, we define the aggregate stability and plasticity metrics evaluated on the new policy's distribution:
\[
M_{{\mathrm{old}}} \triangleq \mathbb{E}_{g_{\mathrm{old}} \sim p_{\mathrm{old}}} \left[ M_{g_{\mathrm{old}}}(\pi^{\mathrm{new}}) \right], \quad
M_{{\mathrm{new}}} \triangleq \mathbb{E}_{g_{\mathrm{new}} \sim p_{\mathrm{new}}} \left[ M_{g_{\mathrm{new}}}(\pi^{\mathrm{new}}) \right].
\]
\end{definition}

\subsection{Proof of Theorem~\ref{thm:unified_kl}}
\label{app:thereom_deri}
We employ the Goal-Conditioned Performance Difference Lemma (PDL) \citep{kakade2002approximately}:
\begin{lemma}[Goal-Conditioned PDL]
\label{lem:gc_pdl}
For any policies $\pi^\mathrm{new}, \pi^\mathrm{old}$ and goal $g$, the performance difference satisfies:
\[
J_g(\pi^\mathrm{new}) - J_g(\pi^\mathrm{old}) = \frac{1}{1-\gamma} \mathbb{E}_{s\sim d_g^{\pi^\mathrm{new}}}\left[ \mathbb{E}_{a\sim\pi^\mathrm{new}(\cdot|s,g)}[A_g^{\pi^\mathrm{old}}(s,a)] \right].
\]
\end{lemma}

\subsubsection{Proof of Stability Bound}
We aim to bound the performance degradation on old tasks: $|J_{\mathrm{old}}(\pi^{\mathrm{new}}) - J_{\mathrm{old}}(\pi^{\mathrm{old}})|$. 
Fix a goal $g \sim p_{\mathrm{old}}$. Let $f_g(s) \triangleq \mathbb{E}_{a\sim\pi^{\mathrm{new}}(\cdot|s,g)}[A_g^{\pi^{\mathrm{old}}}(s,a)]$.
Using Lemma~\ref{lem:gc_pdl}, we decompose the expectation over $d_g^{\pi^{\mathrm{new}}}$:
\[
J_g(\pi^{\mathrm{new}}) - J_g(\pi^{\mathrm{old}}) = \frac{1}{1-\gamma} \left( \underbrace{\mathbb{E}_{s\sim d_g^{\pi^{\mathrm{old}}}}[f_g(s)]}_{\text{(I) Action Mismatch}} + \underbrace{\sum_s (d_g^{\pi^{\mathrm{new}}}(s) - d_g^{\pi^{\mathrm{old}}}(s)) f_g(s)}_{\text{(II) Occupancy Mismatch}} \right).
\]
\textbf{Bounding the Supremum Term:}
Note that for any state $s \in \mathrm{supp}(d_g^{\pi^\mathrm{new}})$ and action $a \sim \pi^\mathrm{new}(\cdot|s,g)$, the pair $(s,a)$ falls within the support of the visiting policy. By Definition~\ref{def:advantage_magnitude}:
\[
\big| f_g(s) \big| = \big| \mathbb{E}_{a\sim\pi^{\mathrm{new}}}[A_g^{\pi^{\mathrm{old}}}(s,a)] \big| \le \sup_{a} \big| A_g^{\pi^{\mathrm{old}}}(s,a) \big| \le M_g(\pi^{\mathrm{new}}).
\]

\textbf{Term (II) Occupancy Mismatch:}
This term represents the drift in state distribution. Applying H\"older's inequality and the discounted occupancy bound \citep{achiam2017constrained}:
\[
\big| \text{Term (II)} \big| \le \|d_g^{\pi^{\mathrm{new}}} - d_g^{\pi^{\mathrm{old}}}\|_1 \cdot \sup_{s \in \mathrm{supp}(d_g^{\pi^\mathrm{new}})} |f_g(s)|.
\]
Substituting the bounds:
\[
\big| \text{Term (II)} \big| \le \left( \frac{2\gamma}{1-\gamma} \mathbb{E}_{s\sim d_g^{\pi^{\mathrm{old}}}} [D_{\mathrm{TV}}(\pi^{\mathrm{new}} \| \pi^{\mathrm{old}})] \right) \cdot M_g(\pi^{\mathrm{new}}).
\]
Applying Pinsker's inequality ($D_{\mathrm{TV}} \le \sqrt{\frac{1}{2} D_{\mathrm{KL}}}$) and substituting back into the PDL (multiplying by $\frac{1}{1-\gamma}$):
\[
\text{Error}_{\text{occ}} \le \frac{2\gamma}{(1-\gamma)^2} M_g(\pi^{\mathrm{new}}) \cdot \mathbb{E}_{s\sim d_g^{\pi^{\mathrm{old}}}} \left[\sqrt{\frac{1}{2} D_{\mathrm{KL}}(\pi^{\mathrm{new}} \| \pi^{\mathrm{old}})}\right].
\]
Taking the expectation over $g \sim p_{\mathrm{old}}$, applying Jensen's inequality $\mathbb{E}[\sqrt{X}] \le \sqrt{\mathbb{E}[X]}$, and using the definition $D_{\mathrm{old}} \triangleq \sqrt{2\,\mathbb{E}_{s\sim d_{\mathrm{old}}^{\pi^{\mathrm{old}}}}[D_{\mathrm{KL}}]}$:
\[
|J_{\mathrm{old}}(\pi^{\mathrm{new}}) - J_{\mathrm{old}}(\pi^{\mathrm{old}})| \le \frac{2\gamma}{(1-\gamma)^2} \underbrace{\mathbb{E}_{g \sim p_{\mathrm{old}}}[M_g(\pi^{\mathrm{new}})]}_{M_{\mathrm{old}}} \cdot D_{\mathrm{old}}.
\]

\subsubsection{Proof of Plasticity Bound}
For new tasks, we consider the performance gain: $J_{\mathrm{new}}(\pi^{\mathrm{new}}) - J_{\mathrm{new}}(\pi^{\mathrm{old}})$.
Using Lemma~\ref{lem:gc_pdl} with $\pi^{\mathrm{old}}$ as the anchor, for a goal $g \sim p_{\mathrm{new}}$:
\[
J_g(\pi^{\mathrm{new}}) - J_g(\pi^{\mathrm{old}}) = \frac{1}{1-\gamma} \mathbb{E}_{s\sim d_g^{\pi^{\mathrm{new}}}}\left[ \mathbb{E}_{a\sim\pi^{\mathrm{new}}}[A_g^{\pi^{\mathrm{old}}}(s,a)] \right].
\]
Since $\mathbb{E}_{a\sim\pi^{\mathrm{old}}}[A_g^{\pi^{\mathrm{old}}}(s,a)] = 0$, we can rewrite the inner expectation:
\[
\mathbb{E}_{a\sim\pi^{\mathrm{new}}}[A_g^{\pi^{\mathrm{old}}}] = \sum_a (\pi^{\mathrm{new}}(a|s) - \pi^{\mathrm{old}}(a|s)) A_g^{\pi^{\mathrm{old}}}(s,a).
\]
By H\"older's inequality:
\[
\big| \mathbb{E}_{a\sim\pi^{\mathrm{new}}}[A_g^{\pi^{\mathrm{old}}}] \big| \le 2 D_{\mathrm{TV}}(\pi^{\mathrm{new}} \| \pi^{\mathrm{old}}) \cdot \sup_{a} \big| A_g^{\pi^{\mathrm{old}}}(s,a) \big|.
\]
For states $s$ visited by $\pi^{\mathrm{new}}$, the advantage is bounded by $M_g(\pi^{\mathrm{new}})$. Applying Pinsker's inequality:
\[
\mathbb{E}_{a\sim\pi^{\mathrm{new}}}[A_g^{\pi^{\mathrm{old}}}] \le M_g(\pi^{\mathrm{new}}) \sqrt{2 D_{\mathrm{KL}}(\pi^{\mathrm{new}} \| \pi^{\mathrm{old}})}.
\]
Substituting back into the PDL:
\[
J_g(\pi^{\mathrm{new}}) - J_g(\pi^{\mathrm{old}}) \le \frac{1}{1-\gamma} M_g(\pi^{\mathrm{new}}) \mathbb{E}_{s\sim d_g^{\pi^{\mathrm{new}}}} \left[ \sqrt{2 D_{\mathrm{KL}}(\pi^{\mathrm{new}} \| \pi^{\mathrm{old}})} \right].
\]
Taking the expectation over $g \sim p_{\mathrm{new}}$, applying Jensen's inequality, and using the definitions of $M_{\mathrm{new}}$ and $D_{\mathrm{new}}$:
\[
J_{\mathrm{new}}(\pi^{\mathrm{new}}) - J_{\mathrm{new}}(\pi^{\mathrm{old}}) \le \frac{1}{1-\gamma} M_{\mathrm{new}} \cdot D_{\mathrm{new}}.
\]
\qed

\subsection{Corollaries and Proofs}
\label{app:corollaries}

\subsubsection{Bounds on Advantage Magnitude}
The metric $M_g(\pi^{\mathrm{new}})$ defined in Definition~\ref{def:advantage_magnitude} depends on the magnitude of the estimated advantage $\hat{A}$. We show that this magnitude is naturally bounded by the critic's approximation error (for old tasks) or the return range (for new tasks).

\begin{corollary}[Stability Control via Critic Error]
For old goals $g \in p_{\mathrm{old}}$, assume rewards are bounded by $R_{\max}$ and the critic approximates the true value with error $\varepsilon$. Then $M_{g}(\pi^{\mathrm{new}})$ is bounded by:
\[
M_{g_{\mathrm{old}}}(\pi^{\mathrm{new}}) \le C_1 R_{\max} + C_2 \varepsilon,
\]
where $C_1, C_2$ are constants depending on $\gamma$.
\end{corollary}
\begin{proof}
(Proof omitted for brevity, follows from triangle inequality on $\hat{A} = r + \gamma V - V$ and bound analysis in Section A.3.1).
\end{proof}

\begin{corollary}[Plasticity Control via Return Range]
For new goals $g \in p_{\mathrm{new}}$, if critics are fitted to MC returns bounded in $[G_{\min}, G_{\max}]$, let $G_{\mathrm{abs}} = \max(|G_{\min}|, |G_{\max}|)$. Then:
\[
M_{g_{\mathrm{new}}}(\pi^{\mathrm{new}}) \le 2 (1+\gamma) G_{\mathrm{abs}}.
\]
\end{corollary}
\begin{proof}
Since $\hat{A}$ is evaluated on trajectories collected by $\pi^{\mathrm{new}}$ (the visiting policy), the values satisfy the MC bounds. The proof follows directly from the triangle inequality $|r + \gamma V' - V| \le |r| + \gamma|V'| + |V|$ and substituting $|V| \le G_{\mathrm{abs}}$.
\end{proof}
\subsection{Corollaries and Proofs: Controlling Advantage Magnitude}
\label{app:corollaries}
In Theorem~\ref{thm:unified_kl}, the stability and plasticity bounds depend linearly on the advantage magnitude $M_g(\pi^\mathrm{new})$. Here, we derive explicit bounds for $M_g$ based on the critic's approximation error (for old tasks) and the MC return range (for new tasks). These corollaries theoretically justify why minimizing critic error preserves stability and why bounding MC returns ensures safe plasticity.

\subsubsection{Bounds for Old Tasks (Stability)}

For old tasks, the advantage $\hat{A}_{g_\mathrm{old}}$ is estimated using learned value functions. We denote the true value functions of the anchored policy as $V_{\mathrm{old}}$ and $Q_{\mathrm{old}}$.

\begin{corollary}[V-only Path Bound]
\label{cor:v_only_bound}
Assume the reward is bounded by $|r(s,a,g)| \le R_{\max}$. Let the learned value function $V_\theta$ approximate the anchored policy's value $V_{\mathrm{old}}$ with a uniform error bound $\sup_{s} |V_\theta(s,g) - V_{\mathrm{old}}(s,g)| \le \varepsilon_V$. Let $\|V_{\mathrm{old}}\|_\infty$ denote the infinity norm of the true value function. Then:
\[
M_{g_{\mathrm{old}}} \le R_{\max} + (1+\gamma)\left( \|V_{\mathrm{old}}\|_\infty + \varepsilon_V \right).
\]
\end{corollary}

\begin{proof}
Under the V-only path, the advantage is estimated as $\hat{A}(s,a) = r(s,a) + \gamma \mathbb{E}_{s'}[V_\theta(s')] - V_\theta(s)$. By the triangle inequality and Jensen's inequality ($|\mathbb{E}[X]| \le \mathbb{E}[|X|]$):
\begin{align*}
|\hat{A}(s,a)| &\le |r(s,a)| + \gamma \mathbb{E}_{s'}[|V_\theta(s')|] + |V_\theta(s)|.
\end{align*}
Using the approximation assumption, for any state $s$, $|V_\theta(s)| \le |V_{\mathrm{old}}(s)| + \varepsilon_V \le \|V_{\mathrm{old}}\|_\infty + \varepsilon_V$. Substituting these bounds:
\begin{align*}
|\hat{A}(s,a)| &\le R_{\max} + \gamma (\|V_{\mathrm{old}}\|_\infty + \varepsilon_V) + (\|V_{\mathrm{old}}\|_\infty + \varepsilon_V) \\
&= R_{\max} + (1+\gamma)(\|V_{\mathrm{old}}\|_\infty + \varepsilon_V).
\end{align*}
Since this bound holds for all $(s,a)$, it holds for the supremum over the support of any visiting policy $\pi^\mathrm{new}$, thus proving the corollary.
\end{proof}

\begin{corollary}[Q-only Path Bound]
\label{cor:q_only_bound}
Assume the learned Q-function $Q_\phi$ approximates the anchored Q-function $Q_{\mathrm{old}}$ with error $\sup_{s,a} |Q_\phi(s,a,g) - Q_{\mathrm{old}}(s,a,g)| \le \varepsilon_Q$. Let $\|Q_{\mathrm{old}}\|_\infty$ be the supremum of the true Q-values. Then:
\[
M_{g_{\mathrm{old}}} \le 2\left( \|Q_{\mathrm{old}}\|_\infty + \varepsilon_Q \right).
\]
\end{corollary}

\begin{proof}
The Q-only advantage is $\hat{A}(s,a) = Q_\phi(s,a) - \mathbb{E}_{a'\sim\pi}[Q_\phi(s,a')]$. By the triangle inequality:
\[
|\hat{A}(s,a)| \le |Q_\phi(s,a)| + \sup_{a'} |Q_\phi(s,a')| \le 2 \sup_{a'} |Q_\phi(s,a')|.
\]
Using the error bound $|Q_\phi(s,a)| \le \|Q_{\mathrm{old}}\|_\infty + \varepsilon_Q$, we obtain:
\[
|\hat{A}(s,a)| \le 2(\|Q_{\mathrm{old}}\|_\infty + \varepsilon_Q).
\]
This completes the proof.
\end{proof}

\subsubsection{Bounds for New Tasks (Plasticity)}

For new tasks, where no pre-trained critic exists, the advantage is estimated using MC returns $G_t$.

\begin{corollary}[MC Estimation Controls Plasticity]
\label{cor:mc_bound}
Assume the value functions are fitted to MC returns bounded within $[G_{\min}, G_{\max}]$. Let $G_{\mathrm{abs}} \triangleq \max(|G_{\min}|, |G_{\max}|)$. Then the advantage magnitude $M_{g_{\mathrm{new}}}$ is bounded by:
\begin{enumerate}
    \item \textbf{V-only path:} $M_{g_{\mathrm{new}}} \le 2(1+\gamma) G_{\mathrm{abs}}$.
    \item \textbf{Q-only path:} $M_{g_{\mathrm{new}}} \le 2 G_{\mathrm{abs}}$.
\end{enumerate}
\end{corollary}

\begin{proof}
Since the critics are trained on bounded returns, we have $|V(s)| \le G_{\mathrm{abs}}$ and $|Q(s,a)| \le G_{\mathrm{abs}}$ for all states visited.

\textbf{1. V-only path:}
Recall that the realized return satisfies $G_t = r_t + \gamma G_{t+1}$. Rearranging for reward implies $r_t = G_t - \gamma G_{t+1}$, thus $|r_t| \le |G_t| + \gamma |G_{t+1}| \le (1+\gamma)G_{\mathrm{abs}}$.
Substituting this reward bound and the value bound into the advantage definition:
\begin{align*}
|\hat{A}(s,a)| &\le |r| + \gamma |V(s')| + |V(s)| \\
&\le (1+\gamma)G_{\mathrm{abs}} + \gamma G_{\mathrm{abs}} + G_{\mathrm{abs}} \\
&= 2(1+\gamma)G_{\mathrm{abs}}.
\end{align*}

\textbf{2. Q-only path:}
Directly applying the triangle inequality to $\hat{A}(s,a) = Q(s,a) - \mathbb{E}[Q(s,\cdot)]$:
\[
|\hat{A}(s,a)| \le |Q(s,a)| + \sup_{a'} |Q(s,a')| \le G_{\mathrm{abs}} + G_{\mathrm{abs}} = 2 G_{\mathrm{abs}}.
\]
\end{proof}

\subsubsection{Summary of Theoretical Implications}
These corollaries provide the mechanism for the Unified Stability-Plasticity bounds in Theorem~\ref{thm:unified_kl}:
\begin{itemize}
    \item \textbf{Stability:} Corollaries~\ref{cor:v_only_bound} and~\ref{cor:q_only_bound} imply that stability on old tasks is governed by the quality of value function approximation ($\varepsilon_V, \varepsilon_Q$). A well-optimized critic minimizes $M_{\mathrm{old}}$, thereby tightening the trust region and preventing catastrophic forgetting.
    \item \textbf{Plasticity:} Corollary~\ref{cor:mc_bound} ensures that even during the exploration of new tasks (where $\pi^\mathrm{new}$ may diverge significantly from $\pi^\mathrm{old}$), the advantage magnitude remains strictly bounded by the return range. This prevents destabilizing updates solely due to estimation variance, facilitating safe adaptation.
\end{itemize}

\subsection{Objectives and Total Training Loss}
\label{Objectives and Total Training Loss}
\paragraph{Policy Divergence Constraint for $D_{\mathrm{new}}$ and $D_{\mathrm{old}}$}
The trust region in Eq.~\ref{PPO_loss} bounds $D_{\mathrm{new}}$. To constrain $D_{\mathrm{old}}$ per Theorem~\ref{thm:unified_kl}, we add goal-conditioned behavior cloning on $\mathcal{B}_{\text{old}}$:
\begin{equation}
\label{KL_constrain}
\mathcal{L}_{\mathrm{KL}}(\theta) = \mathbb{E}_{(s,a,g)\sim\mathcal{B}_{\text{old}}} \left[ \|\pi_\theta(\cdot\mid s,g) - a\|^2 \right].
\end{equation}

\paragraph{Total Training Loss: Lagrangian Relaxation}
Recall our constrained continual learning objective:
\begin{equation}
\begin{aligned}
\max_{\pi^k} &\quad J_{g_k}(\pi^k; \mathcal{T}_k) \\
\text{s.t.} &\quad J_{g_{k-1}}(\pi^k; \mathcal{T}_{k-1}) \ge J_{g_{k-1}}(\pi^{k-1}; \mathcal{T}_{k-1}) - \delta.
\end{aligned}
\label{eq:constrained_obj}
\end{equation}
We convert this into an unconstrained optimization via Lagrangian relaxation. By Theorem~\ref{thm:unified_kl}, the constraint can be decomposed into bounds on policy divergence ($D_{\mathrm{old}}, D_{\mathrm{new}}$) and value accuracy ($M_{\mathrm{old}}, M_{\mathrm{new}}$). This yields the Lagrangian:
\begin{equation}
\begin{aligned}
\mathcal{L}(\pi, \alpha, \beta_V, \eta) = &\underbrace{-J_{g_k}(\pi^k)}_{\text{maximize new return}} \\
&+ \alpha \cdot \underbrace{D_{\mathrm{old}}(\pi^k, \pi^{k-1})}_{\text{policy divergence penalty}} \\
&+ \beta_V \cdot \underbrace{M_{\mathrm{old}}}_{\text{old value error penalty}} \\
&+ \eta \cdot \underbrace{M_{\mathrm{new}}}_{\text{new value error penalty}},
\end{aligned}
\end{equation}
where $\alpha, \beta_V, \eta \ge 0$ are Lagrange multipliers enforcing the constraint.

We instantiate this Lagrangian with practical loss terms:
\begin{equation}
\mathcal{L}_{\mathrm{total}} = \mathcal{L}_{\mathrm{PPO}} + \alpha\,\mathcal{L}_{\mathrm{KL}} + \beta_V\,\mathcal{L}^{V}_{\mathrm{GCV}} + \eta\,\mathcal{L}^{V}_{\mathrm{MC}},
\label{eq:loss_total}
\end{equation}
where:
\begin{itemize}
    \item $\mathcal{L}_{\mathrm{PPO}} \approx -J_{g_k}(\pi^k)$: PPO objective maximizing new-task return with $D_{\mathrm{new}}$ trust region;
    \item $\mathcal{L}_{\mathrm{KL}} \propto D_{\mathrm{old}}$: behavior cloning loss bounding old-task policy divergence;
    \item $\mathcal{L}^{V}_{\mathrm{GCV}} \propto M_{\mathrm{old}}$: frozen critic loss constraining old-task value drift;
    \item $\mathcal{L}^{V}_{\mathrm{MC}} \propto M_{\mathrm{new}}$: trainable critic loss ensuring new-task value accuracy.
\end{itemize}
Thus, Eq.~\eqref{eq:loss_total} is a Lagrangian relaxation of Eq.~\eqref{eq:constrained_obj}, with hyperparameters $\alpha, \beta_V, \eta$ serving as Lagrange multipliers that balance stability (preserving old-task performance) and plasticity (learning new tasks).

\section{Implementation Details}
\label{appendix:impl}

\subsection{Architecture and Value Function Design}

\paragraph{Goal-Conditioned Value Formulation}
The value network is conditioned on language goal $g$ by concatenating language token embeddings with state representations (extracted from the VLA backbone) before feeding through an MLP. This design ensures sensitivity to language-conditioned goals, inheriting smoothness from the shared embedding geometry. After task transition, the value network's sensitivity to state actions automatically adjusts through the goal-conditioned mechanism.

\paragraph{Dual-Critic Implementation.}
The policy consists of three critics on shared backbone $\phi$:
\begin{itemize}
\item Action critic $\omega$.
\item MC critic $\theta_{\text{MC}}$: trained on new task returns, initialized randomly for task 1
\item GCV critic $\theta_{\text{GCV}}$: copied from previous task's MC critic at task transition, then frozen
\end{itemize}

At task $k=1$, standard RL training optimizes $\mathcal{L}_{\mathrm{PPO}} + \eta\,\mathcal{L}^{V}_{\mathrm{MC}}$. Upon transition to task $k>1$, we freeze the converged MC critic as the new GCV critic and initialize a fresh MC critic.


\paragraph{PPO Loss Implementation and KL Constraint}
\label{PPO appendix}
Following \citet{schulman2017proximal}, we optimize the clipped surrogate objective for each task to ensure stable policy updates. Specifically, we constrain the policy update by enforcing a KL divergence penalty between the updated policy $\pi_\theta$ and the reference policy $\pi'$, thereby adhering to the trust region principle.
Building upon this framework, the surrogate constraint for policy divergence constrains $D_{\mathrm{new}}$. The overall PPO loss is then expressed as:
\begin{equation}
\label{PPO_loss}
\mathcal{L}_{\text{PPO}}(\theta) = \mathbb{E}_{(s,a,g)\sim\mathcal{B}_{\mathrm{new}}}\Bigg[ 
    \min\bigg( r_{\theta}(s,a,g) \hat{A}_\theta(s,a,g), 
    \text{clip}\big(r_{\theta}(s,a,g), 1-\epsilon, 1+\epsilon\big) \hat{A}_\theta(s,a,g) \bigg) \Bigg],
\end{equation}
where $\epsilon$ is the \textbf{clipping hyperparameter} that restricts the probability ratio $r_{\theta}(s,a,g) = \frac{\pi_\theta(a|s,g)}{\pi_{\mathrm{old}}(a|s,g)}$. To strictly enforce the trust region, we implement an \textbf{early stopping criterion} based on the approximate KL divergence:
\begin{equation}
\label{PPO_KL_constraint}
D_{\mathrm{KL}}(\pi' \| \pi_\theta) \approx \mathbb{E}_{(s,g)\sim\mathcal{B}_{\mathrm{new}}} \left[ \log \frac{\pi'(a|s,g)}{\pi_\theta(a|s,g)} \right] \leq d_{\text{targ}},
\end{equation}
where $d_{\text{targ}}$ is the \textbf{target KL threshold}. The advantage $\hat{A}_\theta$ is computed using Generalized Advantage Estimation (GAE) \citep{schulman2015high}:
\begin{equation}
\label{GAE}
\hat{A}_t = \sum_{k=0}^{\infty} (\gamma \lambda)^k \delta_{t+k}^V,
\end{equation}
where $\gamma \in [0, 1]$ is the \textbf{discount factor}, $\lambda \in [0, 1]$ is the \textbf{GAE smoothing parameter}, and $\delta_t^V = r_t + \gamma V(s_{t+1}) - V(s_t)$ denotes the TD error.
\subsection{Action-Value (Q-based) Implementation}
\label{Q_based Implementation}
We extend the Dual-Critic Architecture and Goal-Conditioned Value Formulation detailed in the methodology to action-value functions. This extension applies the same principles of independent dimension management to $Q(s,a,g)$, enabling precise control over stability and plasticity in the action-value space.

\paragraph{Goal-Conditioned Action-Value Formulation}
Consistent with the value function design, we leverage the shared Vision-Language Model's embedding as the global state representation $\phi(s,g)$. To parameterize the action-value function $Q(s,a,g)$, we concatenate the continuous action vector $a$ with the language-grounded state embedding $\phi(s,g)$ before the Multi-Layer Perceptron processing. 
This joint encoding ensures that $Q$-values inherit the metric smoothness and semantic understanding from the pretrained VLM embedding space, addressing language-goal tracking issues while incorporating action information.

\paragraph{Dual-Critic Architecture via Action State Value}
To operationalize the control of $M_{\mathrm{old}}$ and $M_{\mathrm{new}}$ within the Q-learning framework, we employ the Dual-Critic Architecture with critic decoupling:
\textbf{1) Frozen GCV Q-critic $\theta_{\text{GCV}}$.} Initialized from the old task's converged Q-network and frozen during new task training. It serves as the anchor for stability, providing reference values $Q_{\text{old}}(s,a,g)$ to regularize backbone updates on historical data.
\textbf{2) Trainable MC Q-critic $\theta_{\text{MC}}$.} Updated on new task trajectories using MC returns. This critic remains plastic to capture the reward structure of the new task, naturally bounding $M_{\mathrm{new}}$ via the environment's reward limits.

\paragraph{Q-based Goal-Conditioned Value Consistency (Q-GCV)}
The Q-based GCV mechanism enforces consistency between the current backbone's predictions (via the frozen GCV critic) and the historical Q-values stored in the replay buffer $\mathcal{B}_{\text{old}}$. Unlike the value function which targets MC returns, the Q-function targets the specific state-action values $Q_{\text{old}}$ recorded from previous interactions. The loss is formulated as the Mean Squared Error (MSE):
{\small
\begin{equation}
\label{eq:q_gcv}
\mathcal{L}^Q_{\text{GCV}}(\phi) = \beta_Q \mathbb{E}_{(s,a,g)\sim\mathcal{B}_{\text{old}}} \left[ \left\| Q_{\phi, \theta_{\text{GCV}}}(s,a,g) - Q_{\text{old}} \right\|^2 \right].
\end{equation}
}
This objective constrains the shared backbone $\phi$ to maintain representations that, when passed through $\theta_{\text{GCV}}$, reconstruct the historical action values, thereby minimizing performance degradation on old tasks ($M_{\mathrm{old}}$).

\paragraph{Q-based Monte Carlo Learning}
To exploit the plasticity dimension ($M_{\mathrm{new}}$), the Trainable MC Q critic is optimized to fit the returns $G_t^{g_{\text{new}}}$ from the current policy on the new task:
{\small
\begin{equation}
\mathcal{L}^Q_{\text{MC}}(\phi, \theta_{\text{MC}}) = \mathbb{E}_{(s,a,g)\sim\mathcal{B}_{\text{new}}} \left[ \left\| Q_{\phi, \theta_{\text{MC}}}(s,a,g) - G_t^{g_{\text{new}}} \right\|^2 \right].
\end{equation}
}
This allows the shared backbone to learn features necessary for the new task without being overly inhibited by the frozen GCV critic, as the optimization of $\theta_{\text{MC}}$ drives the embedding towards relevant new structures.

\subsection{Training Procedure and Hyperparameters}
All experiments are conducted based on OpenVLA-OFT~\citep{kim2025fine}, where the official pretrained checkpoints are adopted as the base model and LoRA adapters are applied for efficient task adaptation. Our implementation is built upon the RIPT-VLA codebase~\citep{tan2025interactiveposttrainingvisionlanguageactionmodels}, and all unspecified training configurations strictly follow the settings reported in prior work.
Training is performed with a LoRA rank of 32, and the detailed hyperparameter configuration is summarized in Tab.~\ref{tab_app:hyperparameter}. Unless otherwise stated, all hyperparameters are kept consistent across experiments to ensure fair comparison.
All experiments are executed on the following hardware platform:
\textbf{CPU:} Intel(R) Xeon(R) Platinum 8358 @ 2.60GHz;
\textbf{GPU:} NVIDIA A100-SXM4-80GB.


\begin{table}[ht]
\centering
\caption{
More performance comparison on the multi-task learning scenario. 
Specific task settings in Appendix~\ref{para:task3}.
$\uparrow$ denotes higher is better, while $\downarrow$ denotes lower is better. The \textbf{best} and \underline{second best} results are highlighted accordingly.}
\label{tab:object2_results}
\begin{tabular}{lccc}
\toprule
\textbf{Method} & \textbf{FAR} $(\uparrow)$ & \textbf{BWT} $(\uparrow)$ & \textbf{FT} $(\uparrow)$ \\
\midrule
SL~\citep{liu2023libero}                & \underline{0.60} & \underline{0.00} & \underline{-0.70} \\
MTL~\citep{liu2023libero}               & 0.54          & \textbf{0.17}   & \underline{-0.70} \\
ER~\citep{lopez2017gradient}             & 0.52         & -0.05           & \textbf{-0.65} \\
LWF  ~\citep{li2017learning}           & 0.43        & \underline{0.00} & \textbf{-0.65} \\
\midrule
CRL-VLA (V)     & \textbf{0.63} & -0.10           & \underline{-0.70} \\
CRL-VLA (Q)     & 0.33          & -0.45           & \underline{-0.70} \\
\bottomrule
\end{tabular}
\end{table}

\begin{table*}[htbp]
\caption{Comparison of continual learning metrics under different parameter settings (Only $\mathcal{L}^{V}_{\mathrm{GCV}}$, Only $\mathcal{L}^{V}_{\mathrm{GCV}}$, and KL), where all experiments are conducted on the Task-4 benchmark (Appendix~\ref{para:task4}). }

\label{tab:results}
\vskip 0.15in
\begin{center}
\begin{small}
\begin{sc}
\begin{tabular}{lcccc}
\toprule
Parameter Value & FAR $(\uparrow)$ & BWT $(\uparrow)$ & Forgetting $(\downarrow)$ & FT $(\uparrow)$ \\
\midrule
\textbf{$\mathcal{L}^{Q}_{\mathrm{GCV}}$} & & & & \\
\midrule
0.00001 & \underline{0.97} & -0.02 & \underline{0.02} & -0.13 \\
0.0001  & 0.94 & -0.05 & 0.05 & \underline{-0.09} \\
0.001   & 0.93 & -0.05 & 0.05 & \textbf{-0.06} \\
0.01    & \textbf{0.98} & \textbf{0.02} & \textbf{0.00} & \underline{-0.09} \\
0.1     & 0.93 & \underline{0.00} & \underline{0.02} & \textbf{-0.06} \\
\midrule
\textbf{$\mathcal{L}^{V}_{\mathrm{GCV}}$} & & & & \\
\midrule
0.00001 & 0.91 & -0.05 & 0.06 & -0.06 \\
0.0001  & 0.91 & \textbf{0.00} & \textbf{0.00} & -0.19 \\
0.001   & \textbf{0.98} & \textbf{0.00} & \textbf{0.00} & -0.06 \\
0.01    & \underline{0.94} & -0.06 & 0.06 & \textbf{-0.03} \\
0.1     & 0.93 & \underline{-0.03} & \underline{0.02} & \underline{-0.06} \\
\midrule
\textbf{KL} & & & & \\
\midrule
0        & 0.91 & -0.03 & 0.03 & \textbf{-0.03} \\
0.000001 & 0.95 & \underline{-0.02} & \textbf{0.02} & -0.09 \\
0.00001  & \textbf{0.98} & \underline{-0.02} & \textbf{0.02} & \underline{-0.06} \\
0.0001   & 0.95 & \underline{-0.02} & \textbf{0.02} & -0.11 \\
0.001    & 0.91 & -0.05 & 0.05 & \textbf{-0.03} \\
0.01     & \underline{0.96} & \textbf{0.02} & \textbf{0.02} & \underline{-0.06} \\
0.1      & 0.90 & \underline{-0.02} & 0.06 & -0.09 \\
\bottomrule
\end{tabular}
\end{sc}
\end{small}
\end{center}
\vskip -0.1in
\end{table*}

\begin{table*}[htbp]
\caption{Ablation study on the effects of the $\mathcal{L}^{V}_{\mathrm{MC}}$ loss coefficient and the $\mathcal{L}^{Q}_{\mathrm{MC}}$ loss coefficient on performance and stability. All experiments are conducted on the Task-4 benchmark (Appendix~\ref{para:task4}).}
\label{tab:ablation_study}
\vskip 0.15in
\begin{center}
\begin{small}
\begin{sc}
\begin{tabular}{lcccc}
\toprule
Parameter Coef & FAR $(\uparrow)$ & BWT $(\uparrow)$ & Forgetting $(\downarrow)$ & FT $(\uparrow)$ \\
\midrule
\textbf{$\mathcal{L}^{V}_{\mathrm{MC}}$} & & & & \\
\midrule
10    & \underline{0.93} & \textbf{0.00}  & \textbf{0.02} & \underline{-0.06} \\
1     & \textbf{0.95}    & \underline{-0.02} & \textbf{0.02} & \textbf{-0.03} \\
0.1   & \underline{0.93} & -0.03          & 0.03          & \underline{-0.06} \\
0.01  & 0.90             & -0.05          & 0.05          & -0.13 \\
0.001 & 0.85             & -0.08          & 0.08          & -0.13 \\
\midrule
\textbf{$\mathcal{L}^{Q}_{\mathrm{MC}}$} & & & & \\
\midrule
10    & \underline{0.98} & \textbf{0.05} & \textbf{0.00} & \underline{-0.06} \\
1     & \textbf{0.99}    & \textbf{0.05} & \textbf{0.00} & \textbf{0.00} \\
0.1   & 0.94             & \underline{0.03} & \textbf{0.00} & \underline{-0.06} \\
0.01  & 0.85             & -0.06          & 0.09          & \underline{-0.06} \\
0.001 & 0.92             & -0.03          & \underline{0.08} & \underline{-0.06} \\
\bottomrule
\end{tabular}
\end{sc}
\end{small}
\end{center}
\vskip -0.1in
\end{table*}

\begin{table*}[htbp]
\centering
\caption{Ablation study of the Goal-Conditioned Value (GCV) network on model performance and stability. All experiments are conducted on the Task-4 benchmark (Appendix~\ref{para:task4}). Enabling GCV leads to higher success rates (FAR) and significantly reduced forgetting.}
\label{tab:gcv_ablation}
\begin{small}
\begin{sc}
\begin{tabular}{lcccc}
\toprule
Used GCV & FAR $(\uparrow)$ & BWT $(\uparrow)$ & Forgetting $(\downarrow)$ & Forward Transfer $(\uparrow)$ \\
\midrule
False & 0.29 & -0.06 & 0.06 & -0.72 \\
True  & \textbf{0.31} & \textbf{-0.02} & \textbf{0.02} & \textbf{-0.70} \\
\bottomrule
\end{tabular}
\end{sc}
\end{small}
\end{table*}


\section{Task and Experiment Details}
\label{appendix:Task and Experiment Details}
To systematically evaluate the robustness and scalability of our method, we construct multiple benchmarks by randomly sampling tasks from a shared task pool. Each benchmark contains a different number of object manipulation tasks, serving as controlled evaluation settings with varying task complexity and combinatorial difficulty. All benchmarks are fixed across methods and runs to ensure fair comparison.

Specifically, we evaluate our method on a subset of the \textbf{LIBERO} benchmark suite, consisting of tasks randomly sampled from the original collection to encompass various task cardinalities. We denote these benchmarks as \textbf{Task-1} through \textbf{Task-4}. Each benchmark comprises multiple task sets defined by natural language directives, primarily focusing on robotic pick-and-place behaviors. This setup allows us to rigorously assess the model's instruction-following capabilities and its generalization across diverse task scales.
we next detail the language instructions used for each benchmark. 
Each task consists of one or more natural language directives.
\subsection{Task-1 Benchmark}\label{para:task1}
\begin{itemize}
    \item \textit{Task 0}: ``pick up the black bowl from table center and place it on the plate''
    \item \textit{Task 1}: ``pick up the black bowl next to the ramekin and place it on the plate''
    \item \textit{Task 2}: ``pick up the black bowl on the cookie box and place it on the plate''
\end{itemize}

\subsection{Task-2 Benchmark}\label{para:task2}
\begin{itemize}
    \item \textit{Task 0}: ``pick up the orange juice and place it in the basket'', ``pick up the chocolate pudding and place it in the basket'', ``pick up the bbq sauce and place it in the basket''
    \item \textit{Task 1}: ``pick up the cream cheese and place it in the basket'', ``pick up the milk and place it in the basket'', ``pick up the tomato sauce and place it in the basket''

\end{itemize}

\subsection{Task-3 Benchmark}\label{para:task3}
\begin{itemize}
    \item \textit{Task 0}: ``pick up the bbq sauce and place it in the basket'', ``pick up the chocolate pudding and place it in the basket''
    \item \textit{Task 1}: ``pick up the ketchup and place it in the basket'', ``pick up the salad dressing and place it in the basket''
\end{itemize}


\subsection{Task-4 Benchmark}\label{para:task4}
\begin{itemize}
    \item \textit{Task 0}: ``put the wine bottle on the rack'', ''put the cream cheese in the bowl''
    \item \textit{Task 1}: ``push the plate to the front of the stove'', ''put the bowl on the stove.''
    \item \textit{Task 2}: ``put the wine bottle on top of the cabinet'', ''open the top drawer and put the bowl inside.''
\end{itemize}

\subsection{Baselines}
\textbf{1)} Sequential Learning (SL)~\citep{liu2023libero}: This is the most straightforward continual learning method, where the VLA model is simply fine-tuned on new tasks without any specific mechanisms to prevent forgetting.
\textbf{2)} Multi-Task Learing (MTL)~\citep{liu2023libero}: This baseline algorithm learns from both new and old tasks simultaneously. 
\textbf{3)} Learning Without Forgetting (LWF)~\citep{li2017learning}: This algorithm preserves existing knowledge without requiring the old data by having the new model \textit{imitate} the old model's predictions on new data during training (knowledge distillation).
\textbf{4)} Experience Replay (ER)~\citep{lopez2017gradient}: It is a classic example of combining reinforcement learning and supervised learning, explicitly using episode memory to store samples from past tasks and converting them into gradient constraints to prevent the gradients of new tasks from degrading the performance on old tasks.

\begin{table*}[htbp]
    \centering
    \caption{Training hyperparameter configuration.}
    \begin{tabular*}{0.65\textwidth}{c @{\extracolsep{\fill}} c}
        \toprule
        \textbf{Parameter} & \textbf{Value} \\
        \midrule
        \multicolumn{2}{c}{\textit{General}} \\
        \midrule
        LoRA Rank & $32$ \\
        Gradient Accumulation Steps & $1$ \\
        PPO Epochs & $1$ \\
        PPO Clip Range & $0.2$ \\
        PPO Clip High & $0.2$ \\
        Max Step Batch Size & $2$ \\
        Learning Rate (LoRA modules) & $1.0 \times 10^{-4}$ \\
        Learning Rate (Action head) & $5.0 \times 10^{-5}$ \\
        Weight Decay & $1.0 \times 10^{-4}$ \\
        Gradient Clip Norm (model) & $1.0$ \\
        Gradient Clip Norm (header) & $1.0$ \\
        Total Steps $M$  & $12$ \\
        Eval Interval $I_{interval}$ & $1$ \\
        Update Times $N$ & $10$ \\
        \midrule
        \multicolumn{2}{c}{\textit{CRL-VLA Weight}} \\
        \midrule
        Q-Network Constraint ($\mathcal{L}^{Q}_{\mathrm{GCV}}$) & $0.01$ \\
        Value-Head Constraint ($\mathcal{L}^{V}_{\mathrm{GCV}}$) & $0.01$ \\
        KL Divergence Scale ($\lambda_{\text{KL}}$) & $1.0 \times 10^{-6}$ \\
    Q Loss Coef ($\mathcal{L}^{Q}_{\mathrm{MC}}$) & $1.0$ \\
        Value Loss Coef ($\mathcal{L}^{V}_{\mathrm{MC}}$) & $1.0$ \\
        
        Global Context Vector (GCV) & Enabled \\
        \bottomrule
    \end{tabular*}
    \label{tab_app:hyperparameter}
\end{table*}

\end{document}